\documentclass{article} % For LaTeX2e
\usepackage{iclr2022_conference,times}

\usepackage{hyperref}
\usepackage{url}
\usepackage{blindtext}

\title{Sqrt(d) Dimension Dependence of Langevin Monte Carlo}

\author{%
  Ruilin Li\\
  Georgia Institute of Technology\\
  \texttt{ruilin.li@gatech.edu} \\
  \And
   Hongyuan Zha \\
   School of Data Science\\
   Shenzhen Institute of Artificial Intelligence and Robotics for Society\\ The Chinese University of Hong Kong, Shenzhen \\
   \texttt{zhahy@cuhk.edu.cn} \\
   \AND
   Molei Tao \\
   Georgia Institute of Technology \\
   \texttt{mtao@gatech.edu} \\
}

\iclrfinalcopy % Uncomment for camera-ready version, but NOT for submission.

\usepackage{booktabs}       % professional-quality tables
\usepackage{amsfonts}       % blackboard math symbols
\usepackage{nicefrac}       % compact symbols for 1/2, etc.
\usepackage{microtype}      % microtypography
\usepackage{xcolor}         % colors

\usepackage{tablefootnote}
\usepackage{amsmath,amsthm, amssymb}
\usepackage{mathtools}
\usepackage{commath}
\usepackage{bbm}
\usepackage{listings}
\usepackage{subcaption}
\usepackage{float}
\floatstyle{plaintop}
\restylefloat{table}

\DeclarePairedDelimiterX{\innerprod}[2]{\langle}{\rangle}{#1, #2}

\DeclarePairedDelimiter{\ceil}{\lceil}{\rceil}

\newcommand{\bs}{\boldsymbol}
\newcommand{\innerprodA}[2]{\langle #1, #2 \rangle_A}
\newcommand{\Clmc}{C_\text{LMC}}

\newtheorem{theorem}{Theorem}[section]
\newtheorem{lemma}[theorem]{Lemma}

\newtheorem{definition}[theorem]{Definition}

\newtheorem{corollary}[theorem]{Corollary}
\newtheorem{assumption}{A}
\newtheorem*{remark}{Remark}

\renewcommand{\O}[1]{\mathcal{O}(#1)}
\newcommand{\tildeO}[1]{\widetilde{\mathcal{O}}\left(#1\right)}
\renewcommand{\norm}[1]{\left\| #1 \right\|}

\usepackage{tikz}
\newcommand*\circled[1]{\tikz[baseline=(char.base)]{\node[shape=circle,draw,inner sep=2pt] (char) {#1};}}

\begin{document}

\maketitle

\begin{abstract}
This article considers the popular MCMC method of unadjusted Langevin Monte Carlo (LMC) and provides a non-asymptotic analysis of its sampling error in 2-Wasserstein distance. The proof is based on a refinement of mean-square analysis in \citet{li2019stochastic}, and this refined framework automates the analysis of a large class of sampling algorithms based on discretizations of contractive SDEs. Using this framework, we establish an $\tildeO{\nicefrac{\sqrt{d}}{\epsilon}}$ mixing time bound for LMC, without warm start, under the common log-smooth and log-strongly-convex conditions, plus a growth condition on the 3rd-order derivative of the potential of target measures. This bound improves the best previously known $\tildeO{\nicefrac{d}{\epsilon}}$ result and is optimal (in terms of order) in both dimension $d$ and accuracy tolerance $\epsilon$ for target measures satisfying the aforementioned assumptions. Our theoretical analysis is further validated by numerical experiments.
\end{abstract}

\section{Introduction}
\vspace{-6pt}

The problem of sampling statistical distributions has attracted considerable attention, not only in the fields of statistics and scientific computing, but also in machine learning \citep{robert2013monte, andrieu2003introduction, liu2008monte}; for example, how various sampling algorithms scale with the dimension of the target distribution is a popular recent topic in statistical deep learning (see discussions below for references). For samplers that can be viewed as discretizations of SDEs, the idea is to use an ergodic SDE whose equilibrium distribution agrees with the target distribution, and employ an appropriate numerical algorithm that discretizes (the time of) the SDE. The iterates of the numerical algorithm will approximately follow the target distribution when converged, and can be used for various downstream applications such as Bayesian inference and inverse problem \citep{Dashti2017}. One notable example is the Langevin Monte Carlo algorithm (LMC), which corresponds to an Euler-Maruyama discretization of the overdamped Langevin equation. Its study dated back to at least the 90s \citep{roberts1996exponential} but keeps on leading to important discoveries, for example, on non-asymptotics and dimension dependence, which are relevant to machine learning (e.g., \citet{dalalyan2017further, dalalyan2017theoretical, cheng2018sharp, durmus2019analysis, durmus2019high, vempala2019rapid, dalalyan2018sampling, li2019stochastic, erdogdu2020convergence, mou2019improved,lehec2021langevin}). LMC is closely related to SGD too (e.g., \citet{stephan2017stochastic}). Many other examples exist, based on alternative SDEs and/or different discretizations (e.g.,
\citet{dalalyan2018sampling, ma2019there}; % Kinetic Langevin Monte Carlo algorithm (KLMC)
\citet{mou2021high,li2020hessian};
\citet{roberts1998optimal,chewi2020optimal}; %Metropolis-Adjusted Langevin Algorithm (MALA)
\citet{shen2019randomized}).
% some additional papers from the applied math community:
%@article{bou2010long,
%  title={Long-run accuracy of variational integrators in the stochastic context},
%  author={Bou-Rabee, Nawaf and Owhadi, Houman},
%  journal={SIAM Journal on Numerical Analysis},
%  volume={48},
%  number={1},
%  pages={278--297},
%  year={2010},
%  publisher={SIAM}
%}
%
%@article{hutzenthaler2012strong,
%  title={Strong convergence of an explicit numerical method for SDEs with nonglobally Lipschitz continuous coefficients},
%  author={Hutzenthaler, Martin and Jentzen, Arnulf and Kloeden, Peter E and others},
%  journal={Annals of Applied Probability},
%  volume={22},
%  number={4},
%  pages={1611--1641},
%  year={2012},
%  publisher={Institute of Mathematical Statistics}
%}
%
%@article{abdulle2014high,
%  title={High order numerical approximation of the invariant measure of ergodic SDEs},
%  author={Abdulle, Assyr and Vilmart, Gilles and Zygalakis, Konstantinos C},
%  journal={SIAM Journal on Numerical Analysis},
%  volume={52},
%  number={4},
%  pages={1600--1622},
%  year={2014},
%  publisher={SIAM}
%}

%Other examples include Kinetic Langevin Monte Carlo algorithm (KLMC) \citep{dalalyan2018sampling, ma2019there} and Metropolis-Adjusted Langevin Algorithm (MALA) \citep{roberts1998optimal}, etc.

Quantitatively characterizing the non-asymptotic sampling error of numerical algorithms is usually critical for choosing the appropriate algorithm for a specific downstream application, for providing practical guidance on hyperparameter selection and experiment design, and for designing improved samplers. A powerful tool that dates back to \citep{jordan1998variational} is a paradigm of non-asymptotic error analysis, namely to view sampling as optimization in probability space, and it led to many important recent results (e.g., \cite{liu2016stein, dalalyan2017further, wibisono2018sampling, zhang2018policy, frogner2018approximate, chizat2018global, chen2018unified, ma2019there, erdogdu2020convergence}). It works by choosing an objective functional, typically some statistical distances/divergences, and showing that the law of the iterates of sampling algorithms converges in that objective functional. However, the choice of the objective functional often needs to be customized for different sampling algorithms. For example, KL divergence works for LMC \citep{cheng2018convergence}, but a carefully hand-crafted cross term needs to be added to KL divergence for analyzing KLMC \citep{ma2019there}. Even for the same underlying SDE, different discretization schemes exist and lead to different sampling algorithms, and the analyses of them had usually been case by case (e.g., \citet{cheng2017underdamped, dalalyan2018sampling, shen2019randomized}). Therefore, it would be a desirable complement to have a unified, general framework to study the non-asymptotic error of SDE-based sampling algorithms. Toward this goal, an alternative approach to analysis has recently started attracting attention, namely to resort to the numerical analysis of SDE integrators (e.g., \citet{milstein2013stochastic, kloeden1992stochastic}) and quantitatively connect the integration error to the sampling error. One remarkable work in this direction is \cite{li2019stochastic}, which will be discussed in greater details later on.

The main tool of analysis in this paper will be a strengthened version (in specific aspects that will be clarified soon) of the result in \cite{li2019stochastic}. Although this analysis framework is rather general and applicable to a broad family of numerical methods that discretize contractive\footnote{possibly after a coordinate transformation} SDEs, the main innovation focuses on a specific sampling algorithm, namely LMC, which is widely used in practice. Its stochastic gradient version is implemented in common machine learning systems, such as Tensorflow \citep{abadi2016tensorflow}, and is the off-the-shelf algorithm for large scale Bayesian inference. With the ever-growing size of parameter space, the non-asymptotic error of LMC is of central theoretical and practical interest, in particular, its dependence on the dimension of the sample space. The best current known upper bound of the mixing time in 2-Wasserstein distance for LMC is $\tildeO{\frac{d}{\epsilon}}$ \citep{durmus2019high}. Motivated by a recent result \citep{chewi2020optimal} that shows better dimension dependence for a Metropolis-Adjusted improvement of LMC, we will investigate if the current bound for (unadjusted) LMC is tight, and if not, what is the optimal dimension dependence.

% Many approaches have been proposed to this end, for example, Lyapunov analysis \citep{cheng2018convergence, ma2019there}, viewing sampling as optimization in probability space \citep{cheng2018convergence, durmus2019analysis} and numerical analysis \citep{dalalyan2017theoretical, durmus2016sampling, dalalyan2018sampling, cheng2017underdamped}. These studies lead to fruitful results on non-asymptotic error bounds with explicit dependence on various parameters of the underlying SDE, e.g. dimension, condition number of the potential function of the target, etc. A classical and powerful framework-- mean-square analysis \citep{milstein2013stochastic} in the study of numerical SDE, however, is largely unexplored for modern non-asymptotic analysis of sampling algorithms. This is partly because traditional mean-square analysis only has bound for finite time, i.e. $t \in [0, T]$ and does not extend to regime $t \to \infty$, in addition, the dependence of the error bound on the parameters of underlying SDE is implicit and it is hard to evaluation the performance of a numerical algorithm in e.g. high-dimension problems.

% This work makes some progress in both directions: it provides an automated procedure for the non-asymptotic analysis of sampling error of SDE discretizations, and when applied to LMC, this procedure gives an $\sqrt{d}$ dependence of mixing time on dimension, which we will also show to be optimal.

\vspace{-6pt}
\paragraph{Our contributions} ~
\vspace{-6pt}
% \tao{this part might need to be shortened; will the newly added paragraph above already suffice?}\li{I can see we may have a space problem, but I am not sure if we should streamline the contribution paragraph which seems to be the most important one. I have removed the entire 'Organization' paragraph to give us more space to work with.}\tao{ok, then the short paragraph immediately above can be removed/shortened.}

% The mixing time bound not only reveals the dependence on tolerance $\epsilon$, but also, somewhat surprisingly, shows that the dependence on the parameters of the underlying SDE is also determined by the order of local strong error.

The main contribution of this work is an improved $\tildeO{ \frac{\sqrt{d}}{\epsilon}}$ mixing time upper bound for LMC in 2-Wasserstein distance, under reasonable regularity assumptions. More specifically, we study LMC for sampling from a Gibbs distribution $d\mu \propto \exp\left( -f(\bs{x})\right) d\bs{x}$. Under the standard smoothness and strong-convexity assumptions, plus an additional linear growth condition on the third-order derivative of the potential (which also shares connections to popular assumptions in the frontier literature), our bound improves upon the previously best known $\tildeO{ \frac{d}{\epsilon} }$ result \citep{durmus2019high} in terms of dimension dependence. For a comparison, note it was known that discretized \textbf{kinetic} Langevin dynamics can lead to $\sqrt{d}$ dependence on dimension \citep{cheng2018convergence, dalalyan2018sampling} and some believe that it is the introduction of momentum that improves the dimension dependence, but our result shows that discretized overdamped Langevin (no momentum) can also have mixing time scaling like $\sqrt{d}$. In fact, it is important to mention that recently shown was that Metropolis-Adjusted Euler-Maruyama discretization of \textbf{overdamped} Langevin (i.e., MALA) has an optimal dimension dependence of $\tildeO{\sqrt{d}}$ \citep{chewi2020optimal}, while what we analyze here is the \textbf{unadjusted} version (i.e., LMC), and it has the same dimension dependence (note however that our $\epsilon$ dependence is not as good as that for MALA; more discussion in Section \ref{sec:LMC}). We also constructed an example which shows that the mixing time of LMC is at least $\widetilde{\Omega}\left(\frac{\sqrt{d}}{\epsilon}\right)$. Hence, our mixing time bound has the optimal dependence on both $d$ and $\epsilon$, in terms of order, for the family of target measures satisfying those regularity assumptions. Our theoretical analysis is further validated by empirical investigation of numerical examples.

A minor contribution of this work is the error analysis framework that we use. It is based on the classical mean-square analysis \citep{milstein2013stochastic} in numerical SDE literature, however extended from finite time to infinite time. It is a minor contribution because this extension was already pioneered in the milestone work of \cite{li2019stochastic}, although we will develop a refined version. Same as in classical mean-square analysis and in \cite{li2019stochastic}, the final (sampling in this case) error is only half order lower than the order of local strong integration error ($p_2$). This will lead to a $\tildeO{ C^{\frac{1}{p_2 - \frac{1}{2}}} / \epsilon^{\frac{1}{p_2 - \frac{1}{2}}}}$ mixing time upper bound in 2-Wasserstein distance for the family of algorithms, where $C$ is a constant containing various information of the underlying problem, e.g., the dimension $d$. Nevertheless, the following two are \textbf{new} to this paper: (i) We weakened the requirement on local strong and weak errors. More precisely, \cite{li2019stochastic} requires uniform bounds on local errors, but this could be a nontrivial requirement for SDE integrators; the improvement here only requires non-uniform bounds (although establishing the same result consequently needs notably more efforts, these are included in this paper too). (ii) The detailed expressions of our bounds are not the same as those in \cite{li2019stochastic} (even if local errors could be uniformly bounded), and as we are interested in dimension-dependence of LMC, we work out constants and carefully track their dimension-dependences. Bounds and constants in \cite{li2019stochastic} might not be specifically designed for tightly tracking dimension dependences, as the focus of their seminal paper was more on $\epsilon$ dependence; consequently, its general error bound only led to a $\tilde{O}(d)$-dependence in mixing time when applied to LMC (see Example 1 in \cite{li2019stochastic}), whereas our result leads to $\tilde{O}(\sqrt{d})$.

% \paragraph{Organization} The rest of the paper is organized as follows: We introduce some notations and briefly review Langevin Monte Carlo algorithm and classical mean-square analysis for numerical SDE in Section \ref{sec:preliminaries}. We demonstrate how to extend the results of mean-square analysis from finite time to infinite time for sampling algorithms and obtain non-asymptotic error bound in Section \ref{sec:MSA}. In Section \ref{sec:LMC}, we apply the extended mean-square analysis framework to study the Langevin Monte Carlo algorithm. We complement our theoretical analysis with investigateion of numerical examples in Section \ref{sec:numerical}. Finally, we conclude and list some potential future directions in Section \ref{sec:conclusion}.

\section{Preliminaries}\label{sec:preliminaries}
\vspace{-6pt}

% \emph{Please see Theorem \ref{thm:w2-non-asymptotic}, Corollary \ref{corr:ic-msa} and Theorem \ref{thm:ic-lmc} for the the main results of this paper.}

\paragraph{Notation} Use symbol $\bs{x}$ to denote a $d$-dim. vector, and plain symbol $x$ to denote a scalar variable. $\norm{\bs{x}}$ denotes the Euclidean vector norm. A numerical algorithm is denoted by $\mathcal{A}$ and its $k$-th iterate by $\bar{\bs{x}}_k$. Slightly abuse notation by identifying measures with their density function w.r.t. Lebesgue measure. Use the convention $\tildeO{\cdot} = \O{\cdot} \log^{\O{1}} (\cdot)$ and $\widetilde{\Omega} \left(\cdot\right) = \Omega(\cdot) \log^{\O{1}} (\cdot)$, i.e., the $\tildeO{\cdot}$/$\widetilde{\Omega}(\cdot)$ notation ignores the dependence on logarithmic factors. Use the notation $\widetilde{\Omega}(\cdot)$ similarly. Denote 2-Wasserstein distance by $W_2(\mu_1, \mu_2) = \left(\inf_{(\bs{X}, \bs{Y}) \sim \Pi(\mu_1, \mu_2)} \mathbb{E}\norm{\bs{X} - \bs{Y}}^2\right)^\frac{1}{2}$, where $\Pi(\mu_1, \mu_2)$ is the set of couplings, i.e. all joint measures with $X$ and $Y$ marginals being $\mu_1$ and $\mu_2$. Denote the target distribution by $\mu$ and the law of a random variable $\bs{X}$ by $\mathrm{Law}(\bs{X})$. Finally, denote the mixing time of an sampling algorithm $\mathcal{A}$ converging to its target distribution $\mu$ in 2-Wasserstein distance by 
$
\tau_{\mathrm{mix}}(\epsilon; W_2; \mathcal{A}) = \inf\{ k \ge 0 | W_2(\mathrm{Law}(\bar{\bs{x}}_k), \mu) \le \epsilon\}.
$

\paragraph{SDE for Sampling}
Consider a general SDE
\begin{equation}\label{eq:sde}
    d\bs{x}_t = \bs{b}(t, \bs{x}_t) dt + \bs{\sigma}(t, \bs{x}_t) d\bs{B}_t
\end{equation}
where $\bs{b} \in \mathbb{R}^d$ is a drift term, $\bs{\sigma} \in \mathbb{R}^{ d \times l}$ is a diffusion coefficient matrix and $\bs{B}_t$ is a $l$-dimensional Wiener process. Under mild condition \citep[Theorem 3.1]{pavliotis2014stochastic}, there exists a unique strong solution $\bs{x}_t$ to Eq. \eqref{eq:sde}. Some SDEs admit geometric ergodicity, so that their solutions converge exponentially fast to a unique invariant distribution, and examples include the classical overdamped and kinetic Langevin dynamics, but are not limited to those (e.g., \cite{mou2021high,li2020hessian}). Such SDE are desired for sampling purposes, because one can set the target distribution to be the invariant distribution by choosing an SDE with an appropriate potential, and then solve the solution $\bs{x}_t$ of the SDE and push the time $t$ to infinity, so that (approximate) samples of the target distribution can be obtained.  Except for a few known cases, however, explicit solutions of Eq. \eqref{eq:sde} are elusive and we have to resort to numerical schemes to simulate/integrate SDE. Such example schemes include, but are not limited to Euler-Maruyama method, Milstein methods and Runge-Kutta method (e.g., \cite{kloeden1992stochastic, milstein2013stochastic}). With constant stepsize $h$ and at $k$-th iteration, a typical numerical algorithm takes a previous iterate $\bar{\bs{x}}_{k-1}$ and outputs a new iterate $\bar{\bs{x}}_k$ as an approximation of the solution $\bs{x}_t$ of Eq. \eqref{eq:sde} at time $t = kh$. 

\paragraph{Langevin Monte Carlo Algorithm}
LMC algorithm is defined by the following update rule
\begin{equation}\label{eq:LMC}
    \bar{\bs{x}}_k = \bar{\bs{x}}_{k-1} - h \nabla f(\bar{\bs{x}}_{k-1}) + \sqrt{2h} \bs{\xi}_k, \quad k = 1, 2, \cdots
\end{equation}
where $\left\{\bs{\xi}_k\right\}_{k\in \mathbb{Z}_{>0}}$ are i.i.d. standard $d$-dimensional Gaussian vectors. LMC corresponds to an Euler-Maruyama discretization of the continuous overdamped Langevin dynamics $d\bs{x}_t = -\nabla f(\bs{x}_t) dt + \sqrt{2}d\bs{B}_t$, which converges to an equilibrium distribution $\mu \sim \exp(-f(\bs{x}))$.

\cite{dalalyan2017theoretical} provided a non-asymptotic analysis of LMC. An $\tildeO{\frac{d}{\epsilon^2}}$ mixing time bound in $W_2$ for log-smooth and log-strongly-convex target measures \citep{dalalyan2017further, cheng2018sharp, durmus2019analysis} has been established. It was further improved to $\tildeO{\frac{d}{\epsilon}}$ under an additional Hessian Lipschitz condition \citep{durmus2019high}. Mixing time bounds of LMC in other statistical distances/divergences have also been studied, including total variation distance \citep{dalalyan2017theoretical, durmus2017nonasymptotic} and KL divergence \citep{cheng2018convergence}.

\paragraph{Classical Mean-Square Analysis}
A powerful framework for quantifying the \textit{global} discretization error of a numerical algorithm for Eq. \eqref{eq:sde}, i.e.,
$
    e_k = \left\{\mathbb{E}\norm{\bs{x}_{kh} - \bar{\bs{x}}_k}\right\}^\frac{1}{2},
$
is mean-square analysis (e.g., \cite{milstein2013stochastic}). Mean-square analysis studies how \textit{local} integration error propagate and accumulate into global integration error; in particular, if one-step (local) weak error and strong error (both the exact and numerical solutions start from the same initial value $\bs{x}$) satisfy
\begin{equation}\label{eq:local_error_general_msa}
    \begin{aligned}
            \norm{\mathbb{E}\bs{x}_h - \mathbb{E}\bar{\bs{x}}_1} \le& C_1 \left( 1 + \mathbb{E}\norm{\bs{x}}^2 \right)^\frac{1}{2} h^{p_1}, \quad \mbox{(local weak error)} \\
            \left(\mathbb{E}\norm{\bs{x}_h - \bar{\bs{x}}_1}^2\right)^\frac{1}{2} \le& C_2 \left( 1 + \mathbb{E}\norm{\bs{x}}^2 \right)^\frac{1}{2} h^{p_2}, \quad \mbox{(local strong error)}
    \end{aligned}
\end{equation}
over a time interval $[0, Kh]$ for some constants $C_1, C_2 > 0$, $p_2 \ge \frac{1}{2}$ and $p_1 \ge p_2 + \frac{1}{2}$, then the global error is bounded by $
    e_k \le C \left( 1 + \mathbb{E}\norm{\bs{x}_0}^2 \right)^\frac{1}{2} h^{p_2 - \frac{1}{2}}, \, k=1,\cdots,K$ 
for some constant $C > 0$ dependent on $Kh$.

Although classical mean-square analysis is only concerned with numerical integration error, sampling error can be also inferred. However, there is a limitation that prevents its direct employment in analyzing sampling algorithms: the global error bound only holds in finite time because the constant $C$ can grow exponentially as $K$ increases, rendering the bound useless when $K \to \infty$. 

% Second, the classical mean-square analysis does not keep track of the dependence of $C$ on various parameters, e.g. dimension, condition number of the potential function of the target distribution, which are the main focus of current research in the field.

\section{Mean-Square Analysis of Samplers Based on Contractive SDE}\label{sec:MSA}
\vspace{-6pt}
We now review and present some improved results on how to use mean-square analysis of \textbf{integration} error to quantify \textbf{sampling} error. A seminal paper in this direction is \cite{li2019stochastic}. What is known / new will be clarified. In all cases, the first step is to lift the finite time limitation when the SDE being discretized has some decaying property so that local integration errors do not amplify with time.

%More precisely, one bottleneck that prevents the results of classical mean-square analysis from extending to infinite time horizon, is the fact that the solution of a general SDE may not be bounded, and neither is its discretization. Note that local error (Eq. \eqref{eq:local_error_general_msa}) depends on the initial value. To go from local to global error, these `initial' values correspond to iterates of numerical algorithms, which change from iteration to iteration and can be unbounded, hence when accumulated together, it is possible that the global error may blow up.

%Samplers considered here, on the other hand, are based on stochastic differential equations, each of which weakly converges to a limiting distributions. The solution of the underlying converging SDE, as it converges to the invariant measure, gradually inherits boundedness properties from the target measure. Thus, as long as the target measure has bounded 2nd-moment, a sampling algorithm based on a reasonable discretization of the SDE should also have bounded 2nd-moment. Motivated by this observation, we will assume the sampling algorithms we study are based on contractive SDEs, which is a sufficient condition to ensure the underlying SDE converges to a statistical distribution.

The specific type of decaying property we will work with is contractivity (after coordinate transformation). It is a sufficient condition for the underlying SDE to converge to a statistical distribution.
\begin{definition} A stochastic differential equation is contractive if there exists a non-singular constant matrix $A \in \mathbb{R}^{d\times d}$, a constant $\beta > 0$, such that any pair of solutions of the SDE satisfy
\begin{equation}\label{eq:contraction_def}
    \left(\mathbb{E} \norm{A\left(\bs{x}_t - \bs{y}_t\right)}^2\right)^\frac{1}{2} \le \left(\mathbb{E}\norm{A\left(\bs{x} - \bs{y}\right)}^2\right)^\frac{1}{2} \exp(-\beta t),
\end{equation}
%\tao{can we relax the statement a little bit to allow coordinate transformation, so that kinetic Langevin for example also fits? or, add a remark} \li{Unfortunately, this assumption is key to the proof and can not be relaxed. Typically, we will need to adapt everything else to accommodate this. For example, in ULD and HHFR, we need to work with coordinate-transformed dynamics.} \tao{yes, but that's what I meant. for example, can the requirement of eq.\ref{eq:contraction_def} be replaced by
%\begin{equation}
%    \left(\mathbb{E} \norm{A(\bs{x}_t - \bs{y}_t)}^2\right)^\frac{1}{2} \le \norm{A(\bs{x} - \bs{y})} \exp(-\beta t),
%\end{equation}
%for some $A$ independent of $t$?
%}\li{I relaxed this statement and updated the proof.}
where $\bs{x}_t, \bs{y}_t$ are two solutions, driven by the same Brownian motion but evolved respectively from initial conditions $\bs{x}$ and $\bs{y}$. %\tao{make sure other places (proof etc) are updated to reflect the added expectation on the RHS} \li{Updated. All initial values are in expectation now.}
\end{definition}
\begin{remark}
As long as $\bs{b}$ and $\bs{\sigma}$ in \eqref{eq:sde} are not explicitly dependent on time, it suffices to find an arbitrarily small $t_0>0$ and show \eqref{eq:contraction_def} holds for all $t<t_0$.
\end{remark}
\begin{remark}
Sometimes contraction is not easy to establish directly, but can be shown after an appropriate coordinate transformation, see \cite[Proposition 1]{dalalyan2018sampling} for such a treatment for kinetic Langevin dynamics. The introduction of $A$ permits such transformations.
\end{remark}
In particular, overdamped Langevin dynamics, of which LMC is a discretization, is contractive when $f$ is strongly convex and smooth.

% In addition, we assume the sampling algorithms we study are bounded in the following sense:
% \tao{after some thinking, I'm wondering if we could just call it a \textbf{stable} integrator/algorithm/numerical solution instead of `bounded'.}
% \begin{definition}\label{eq:bounded}
% A numerical algorithm $\mathcal{A}$ that discretizes an underlying SDE is bounded, if all of its iterates satisfy $\mathbb{E}\norm{\bar{\bs{x}}_k}^2 \le U, \, k = 0, 1, \cdots$ for some constant $U > 0$. 
% \end{definition}
% \tao{can we add a remark like `this is sometimes provable, for example in the case of LMC, however...'}\li{This is an assumption, will the remark after 'however' make it sound more restrictive?}\tao{i'm not sure if i understood. saved for next discussion}
%\li{The proof has been updated, and no boundedness requirement is needed. So comments regarding it are commented out.}

We now use contractivity to remove the finite time limitation. We first need a short time lemma.
%\tao{use more neutral language, avoid subjective statement like `we manage to bypass...'}\li{updated.}
\begin{lemma}\cite[Lemma 1.3]{milstein2013stochastic}
Suppose $\bs{b}$ and $\bs{\sigma}$ in Eq.\eqref{eq:sde} are Lipschitz continuous. For two solutions $\bs{x}_t, \bs{y}_t$ of Eq. \eqref{eq:sde} starting from $\bs{x}, \bs{y}$ respectively, denote $\bs{z}_t(\bs{x}, \bs{y}) := (\bs{x}_t - \bs{x}) - (\bs{y}_t - \bs{y})$,  then there exist $C_0 > 0$ and $h_0 > 0$ such that
\begin{equation}\label{eq:z}
    \mathbb{E}\norm{\bs{z}_t(\bs{x}, \bs{y})}^2 \le C_0 \norm{\bs{x} - \bs{y}}^2 t, \quad \forall \bs{x},  \bs{y}, \, 0 < t \le h_0.
\end{equation}
\end{lemma}

Then we have a sequence of results that connects statistical property 
with integration property.
We will see that a non-asymptotic \textbf{sampling} error analysis only requires bounding the orders of local weak and strong \textbf{integration} errors (if the continuous dynamics can be shown contractive).

\begin{theorem}\label{thm:global-error} (\textbf{Global Integration Error, Infinite Time Version}) Suppose Eq.\eqref{eq:sde} is contractive with rate $\beta$ and with respect to a non-singular matrix $A \in \mathbb{R}^{d \times d}$, with Lipschitz continuous $\bs{b}$ and $\bs{\sigma}$, and there is a numerical algorithm $\mathcal{A}$ with step size $h$ simulating the solution $\bs{x}_t$ of the SDE, whose iterates are denoted by $\bar{\bs{x}}_k, k=0,1,\cdots$. Suppose there exists $0 < h_0 \leq 1, C_1, C_2 > 0, D_1, D_2 \ge 0, p_1 \ge 1, \frac{1}{2} < p_2 \le p_1 - \frac{1}{2
}$ such that for any $0 < h \leq h_0$, the algorithm $\mathcal{A}$ has, respectively, local weak and strong error of order $p_1$ and $p_2$, defined as
\begin{equation}
\begin{cases}
    \norm{\mathbb{E} \left( \bs{x}_h - \bar{\bs{x}}_1 \right)} \le \left( C_1 + D_1 \sqrt{\mathbb{E}\norm{\bs{x}}^2} \right) h^{p_1}, \\
    \left(\mathbb{E}\norm{\bs{x}_h - \bar{\bs{x}}_1}^2\right)^\frac{1}{2} \le \left( C_2^2 + D_2^2 \mathbb{E}\norm{\bs{x}}^2 \right)^\frac{1}{2} h^{p_2},
\end{cases}\label{eq:local_error}
\end{equation}
where $\bs{x}_h$ solves Eq.\eqref{eq:sde} with any initial value $\bs{x}$ and $\bar{\bs{x}}_1$ is the result of applying $\mathcal{A}$ to $\bs{x}$ for one step. 

If the solution of SDE $\bs{x}_t$ and algorithm $\mathcal{A}$ both start from $\bs{x}_0$, then for $0 < h \leq h_1 \triangleq \min\left\{h_0, \frac{1}{4\beta},  \left(\frac{\sqrt{\beta}}{4\sqrt{2} \kappa_A D_2} \right)^\frac{1}{p_2 - \frac{1}{2}}, \left( \frac{\beta}{8\sqrt{2}\kappa_A^2 (D_1 + C_0 D_2)} \right)^\frac{1}{p_2 - \frac{1}{2}}    \right\}$, the global error $\bs{e}_k$ is bounded as 
\begin{equation}\label{eq:global_error}
    e_k := \left( \mathbb{E} \| \bs{x}_{kh} - \bar{\bs{x}}_k \|^2 \right)^{\frac{1}{2}} \le C h^{p_2 - \frac{1}{2}}, \quad k = 0, 1, 2, \cdots, \qquad \text{where}
\end{equation}
\begin{equation}\label{eq:constant}
C = \frac{2}{\sqrt{\beta}} \kappa_A^2 \left( \frac{C_1 + C_0 C_2 + \sqrt{2}U(D_1 + C_0 D_2) }{\sqrt{\beta}} +  C_2 + \sqrt{2} D_2U \right),
\end{equation}
$C_0$ is from Eq. \eqref{eq:z}, $\kappa_A$ is the condition number of matrix $A$ and $U^2 \triangleq 4 \mathbb{E}\norm{\bs{x}_0}^2 + 6 \mathbb{E}_\mu \norm{\bs{x}}^2$.
%\tao{Ruilin: can it be replaced by $U^2 \triangleq 4 \mathbb{E} \norm{\bs{x}_0}^2 + 5 \mathbb{E}_\mu \norm{\bs{x}}^2$?}\li{Sure! I updated here and also in the proof.}
\end{theorem}
%\tao{note I fixed inconsistency in notations, including changing the exact solution from $x_t$ to $x(t)$. i didn't change other places; revert back if you'd like}\li{Using $\bs{x}(h)$ would cause inconsistency with notations used in the rest of the paper. I think it would be better to just use $\bs{x}_h$. If it causes any inconsistency, feel free to point it out.} \tao{sure, $x_h$ is fine, but should things like $\bar{x}_h$ in eq.\ref{eq:local_error} be reverted back to $\bar{x}_1$?} \li{I think so and I have updated it throughout the paper.}

\begin{remark}[what's new]  Thm.\ref{thm:global-error} refines the seminal results in \cite{li2019stochastic} in the sense that it only requires non-uniform bounds on the local error \eqref{eq:local_error}, whereas \cite{li2019stochastic} requires uniform bounds, i.e., $D_1=D_2=0$ in \eqref{eq:local_error}. Therefore, the refinement we present has wider applicability.

In general, local errors tend to depend on the current step's value, i.e. $D_1 \neq 0, D_2 \neq 0$. Allowing local error bounds to be non-uniform enabled applications such as proving the vanishing bias of mirror Langevin algorithm \citep{li2021mirror}. For a simpler illustration, consider LMC for 1D standard Gaussian target distribution, then we have 
\[
\norm{\mathbb{E} \left( \bs{x}_h - \bar{\bs{x}}_1 \right)} = (e^{-h} - 1 + h) \mathbb{E}\|\bs{x}\| = (\frac{h^2}{2} + o(h^2)) \mathbb{E}\|\bs{x}\|.
\]
One can see that the local error does depend on $\bs{x}$ and is not uniform. Meanwhile, our non-uniform condition still holds because $\left(\frac{h^2}{2} + o(h^2)\right) \mathbb{E}\|\bs{x}\| \leq (\frac{h^2}{2} + o(h^2)) \sqrt{\mathbb{E}\|\bs{x}\|^2}$ (and thus $p_1=2$).
Note if the discretization does converge to a neighborhood of the target distribution, it is possible that $\mathbb{E}\|\bs{x}\|^2$ and/or $\mathbb{E}\|\bs{x}\|$ become bounded near the convergence, and in this case the $D_1$, $D_2$ parts can be absorbed into $C_1$ and $C_2$; however, this `if' clause is exactly what we'd like to prove.

Nevertheless, we state for rigor that the convention $\nicefrac{1}{0} = \infty$ is used when $D_1 = D_2 = 0$. % This is pertinent when a numerical algorithm $\mathcal{A}$, e.g. LMC (Lemma \ref{lemma:bounded-iterates}), produces bounded iterates. In such cases, the initial value in Eq. \eqref{eq:local_error} are iterations of $\mathcal{A}$ and will be bounded, it then can be absorbed into $C_1, C_2$ and we may set $D_1 = D_2 = 0$.
Another remark is, even in this case, our bound has a different expression from the seminal results. We will carefully work out, track, and combine dimension-dependence of constants using our bound.
\end{remark}

% \begin{remark}
% Equation \eqref{eq:z} is more of a lemma than an assumption, and can be shown for stochastic differential equations with Lipschitz-continuous drift and diffusion terms, see, e.g. \cite[Lemma 1.3]{milstein2013stochastic}. We include it as an assumption rather than a lemma simply because we want to make the dependence of the constant $C$ in Equation \eqref{eq:global_error} on $C_0$ explicit.
% \end{remark}

%\tao{I didn't get this. Why can't we state it as a lemma, which introduces $C_0$, and then use $C_0$ in $C$ just as it is now?} \li{If including it as a lemma, It would require additional assumptions that are less pertinent to the general result we would like to show. Furthermore, for some SDEs, we may use their structure to obtain a better $C_0$ than the general $C_0$.}\tao{i'm not sure if i understood. saved for next discussion} \li{It is singled out as a lemma.}

Following Theorem \ref{thm:global-error}, we obtain the following non-asymptotic bound of the sampling error in $W_2$:
\begin{theorem}\label{thm:w2-non-asymptotic}(\textbf{Non-Asymptotic Sampling Error Bound: General Case})
Under the same assumption and with the same notation of Theorem \ref{thm:global-error}, we have
\begin{equation*}
    W_2(\text{Law}(\bar{\bs{x}}_k), \mu)
    \le e^{- \beta kh}  W_2(\text{Law}(\bs{x}_0), \mu) + C h^{p_2 - \frac{1}{2}}, \quad \forall 0 < h \le h_1.
\end{equation*}
\end{theorem}

A corollary of Theorem \ref{thm:w2-non-asymptotic} is a bound on the mixing time of the sampling algorithm:
\begin{corollary}\label{corr:ic-msa}(\textbf{Upper Bound of Mixing Time: General Case})
Under the same assumption and with the same notation of Theorem \ref{thm:global-error}, we have 
\[
    \tau_{\mathrm{mix}}(\epsilon; W_2; \mathcal{A}) \le \max\left\{\frac{1}{\beta h_1},  \frac{1}{\beta} \left(\frac{2C}{\epsilon}\right)^{\frac{1}{p_2 - \frac{1}{2}}} \right\} \log \frac{2 W_2(\text{Law}(\bs{x}_0), \mu)}{\epsilon}
\]
In particular, when high accuracy is needed, i.e., $\epsilon < 2C h_1^{p_2 - \frac{1}{2}}$, we have
\begin{equation}\label{eq:ic-general}
    \tau_{\mathrm{mix}}(\epsilon; W_2; \mathcal{A}) \leq  \frac{(2C)^{\frac{1}{p_2 - \frac{1}{2}}}}{\beta} \frac{1}{\epsilon^{\frac{1}{p_2 - \frac{1}{2}}} } \log \frac{2 W_2(\text{Law}(\bs{x}_0), \mu)}{\epsilon} = \tildeO{\frac{C^{\frac{1}{p_2 - \frac{1}{2}}}}{\beta} \, \frac{1}{\epsilon^{\frac{1}{p_2 - \frac{1}{2}}} }}.
\end{equation}
\end{corollary}
Corollary \ref{corr:ic-msa} states how mixing time depends on the order of local (strong) error (i.e., $p_2$) of a numerical algorithm. The larger $p_2$ is, the shorter the mixing time of the algorithm is, in term of the dependence on accuracy tolerance parameter $\epsilon$. It is important to note that for constant stepsize discretizations that are deterministic on the filtration of the driving Brownian motion and use only its increments, there is a strong order barrier, namely $p_2 \leq 1.5$ \citep{clark1980maximum,ruemelin1982numerical}; however, methods involving multiple stochastic integrals (e.g., \citet{kloeden1992stochastic, milstein2013stochastic,rossler2010runge}) can yield a larger $p_2$, and randomization (e.g., \citet{shen2019randomized}) can possibly break the barrier too.

The constant $C$ defined in Eq. \eqref{eq:global_error} typically contains rich information about the underlying SDE, e.g. dimension, Lipschitz constant of drift and noise diffusion, and the initial value $\bs{x}_0$ of the sampling algorithm. Through $C$, we can uncover the dependence of mixing time bound on various parameters, such as the dimension $d$. This will be detailed for Langevin Monte Carlo in the next section.

It is worth clarifying that once Thm.\ref{thm:global-error} is proved, establishing Theorem \ref{thm:w2-non-asymptotic} and Corollary \ref{corr:ic-msa} is relatively easy. In fact, analogous results have already been provided in \cite{li2019stochastic}, although they also required uniform local errors as consequences of their Thm.1. Nevertheless, we do not claim novelty in Theorem \ref{thm:w2-non-asymptotic} and Corollary \ref{corr:ic-msa} and they are just presented for completeness. Our main refinement is just Thm.\ref{thm:global-error} over Thm.1 in \cite{li2019stochastic}, and the non-triviality lies in its proof.

% As an application of Theorem \ref{thm:w2-non-asymptotic} and Corollary \ref{corr:ic-msa}, we will work with Langevin dynamics and its discretization Langevin Monte Carlo algorithm, and derive tight bounds on its iteration complexity in the next section.

\section{Non-Asymptotic Analysis of Langevin Monte Carlo Algorithm}\label{sec:LMC}
\vspace{-6pt}
We now quantify how LMC samples from Gibbs target distribution $\mu \sim \exp\left(-f(\bs{x})\right)$ that has a finite second moment, i.e., $\int_{\mathbb{R}^d} \norm{\bs{x}}^2 d\mu < \infty$. Assume without loss of generality that the origin is a local minimizer of $f$, i.e. $\nabla f(\bs{0}) = \bs{0}$; this is for notational convenience in the analysis and can be realized via a simple coordinate shift, and it is not needed in the practical implementation. In addition, we assume the following two conditions hold:
\begin{assumption}\label{asp:smooth-and-strongly-convex}
(\textbf{Smoothness and Strong Convexity}) \, Assume $f \in \mathcal{C}^2$ and is $L$-smooth and $m$-strongly-convex, i.e. there exists $0 < m \leq  L$ such that $m I_d  \preccurlyeq  \nabla^2 f(\bs{x})  \preccurlyeq LI_d, \quad \forall \bs{x} \in \mathbb{R}^d$. 
\end{assumption}
Denote the condition number of $f$ by $\kappa \triangleq \frac{L}{m}$. The smoothness and strong-convexity assumption is the standard assumption in the literature of analyzing LMC algorithm \citep{ dalalyan2017further, dalalyan2017theoretical, cheng2018convergence, durmus2019analysis, durmus2019high}.

\begin{assumption}\label{asp:linear_3rd_derivative}
(\textbf{Linear Growth of the 3rd-order Derivative}) \, Assume $f \in \mathcal{C}^3$ and the operator $\nabla (\Delta f)$ grows at most linearly, i.e., there exists a constant $G > 0$ such that
$\norm{\nabla (\Delta f(\bs{x}))} \le G \left( 1 + \norm{\bs{x}} \right)$. 
% For notation convenience, we will relate $G$ to the smoothness coefficient $L$ and write $G = gL^2.$
\end{assumption}
\begin{remark}
The linear growth (at infinity) condition on $\nabla \Delta f$ is actually not as restrictive as it appears, and in some sense even weaker than some classical condition for the existence of solutions to SDE. For example, a standard condition for ensuring the existence and uniqueness of a global solution to SDE is at most a linear growth (at infinity) of the drift \citep[Theorem 3.1]{pavliotis2014stochastic}. If we consider monomial potentials, i.e.,  $f(x) = x^p, p \in \mathbb{N}_+$, then the linear growth condition on $\nabla \Delta f$ is met when $p \le 4$, whereas the classical condition for the existence of solutions holds only when $p \le 2$.
\end{remark}

\begin{remark}
    Another additional assumption, namely Hessian Lipschitz condition, is commonly used in the literature (e.g., \citet{durmus2019high,ma2019there}). It requires the existence of a constant $\tilde{L}$, such that $\|\nabla^2 f(\bs{y}) - \nabla^2 f(\bs{x})\| \leq \tilde{L} \| \bs{y} - \bs{x}\|$. It can be shown that smoothness and Hessian Lipschitzness imply A\ref{asp:linear_3rd_derivative}. Meanwhile, examples that satisfy A\ref{asp:linear_3rd_derivative} but are not Hessian Lipschitz exist, e.g., $f(x)=x^4$, and thus A\ref{asp:linear_3rd_derivative} is not necessarily stronger than Hessian Lipschitzness.
\end{remark}

\begin{remark}
Same as $L$ and $m$ in A\ref{asp:smooth-and-strongly-convex}, we implicitly assume the constant $G$ introduced in A\ref{asp:linear_3rd_derivative} to be independent of dimension. Meanwhile, it is important to note examples for which $G$ depends on the dimension do exist, and this is also true for other regularity constants including not only $L$ and $m$ but also the Hessian Lipschitz constant $\tilde{L}$. This part of the assumption is a strong one.
\end{remark}

\begin{remark}
A\ref{asp:linear_3rd_derivative}, together with Ito's lemma, helps establish an order $p_1=2$ of local weak error for LMC (see Lemma \ref{lemma:local-weak-error}), which enables us to obtain the $\sqrt{d}$ dependence.
\end{remark}

% In addition, for normalization purpose, we assume without loss of generality, the origin is a local minimizer of $f$, i.e. $\nabla f(\bs{0}) = \bs{0}$.

% Denote the solution of overdamped Langevin dynamics by $\bs{x}_t$, when needed, we will use subscripts $\bs{x}_{t_0, \bs{x}_{t_0}}(t)$ to emphasize the dependence on initial value $\bs{x}_{t_0}$ when $t = t_0$ and place time $t$ dependence in parenthesis.

% \paragraph{Theoretical Analysis}\label{sec:theory}
To apply mean-square analysis to study LMC algorithm, we will need to ensure the underlying Langevin dynamics is contractive, which we verify in Section \ref{sec:app-ld} and \ref{sec:app-lmc} in the appendix. In addition, we work out all required constants to determine the $C$ in Eq. \ref{eq:global_error} explicitly in the appendix. With all these necessary ingredients, we now invoke Theorem \ref{thm:w2-non-asymptotic}  and obtain the following result:
\begin{theorem}\label{thm:w2-lmc}(\textbf{Non-Asymptotic Error Bound: LMC})
Suppose Assumption \ref{asp:smooth-and-strongly-convex} and \ref{asp:linear_3rd_derivative} hold. LMC iteration $\bar{\bs{x}}_{k+1} = \bar{\bs{x}}_k - h \nabla f(\bar{\bs{x}}_k) + \sqrt{2h} \xi_k$ satisfies
\begin{equation}\label{eq:non-asymptotic-bound-lmc}
    W_2(\text{Law}(\bar{\bs{x}}_k), \mu)
    \le  e^{- mkh}  W_2(\text{Law}(\bs{x}_0), \mu) + \Clmc h, \quad 0 < h \le \frac{1}{4\kappa L}, k \in \mathbb{N}
\end{equation}
where $\Clmc = \frac{10(L^2 + G)}{m^\frac{3}{2}} \sqrt{2d + m\left(\mathbb{E}\norm{\bs{x}_0}^2 + 1\right)} = \mathcal{O}(\sqrt{d})$.
\end{theorem}

% Note that the discretization error term $\sqrt{2} C_\text{LMC} h$ is tight (up to a constant) in terms of the dependence on the order of step size $h$ and dimension $d$. For example, consider a standard $d$-dimensional Gaussian as our target, then $f(\bs{x}) = \frac{1}{2} \norm{\bs{x}}^2, \bs{x} \in \mathbb{R}^d$ with $L=m=1, G=0$. For simplicity, suppose LMC algorithm starts from $\bs{x}_0 = \bs{0}$. We then have $C_\text{LMC} = (7+2\sqrt{d}) \sqrt{2d}$ and the discretization error is $\sqrt{2} C_\text{LMC} h = (4 + 4\sqrt{2}) \sqrt{d}h$. We have explicit expression for the solution $\bs{x}_t$ of the linear Langevin equation (Ornstein-Unlenbeck process) and the iterates $\bar{\bs{x}}_k$ of LMC
% \begin{align*}
%     \bs{x}_t =& \sqrt{2} \int_0^t \exp\left(-(t-s)\right)d\bs{B}_s \sim \mathcal{N}\left(\bs{0}, (1 - e^{-2t}) I\right), \\
%     \bar{\bs{x}}_k =& \sqrt{2h}\left(\bs{\xi}_k + (1-h)^2\bs{\xi}_{k-1} + \cdots + (1-h)^{2(k-1)} \bs{\xi}_1 \right) \sim \mathcal{N}\left(\bs{0}, \frac{2}{2-h}\left(1 - (1-h)^{2k}\right) I  \right).
% \end{align*}
% Then the discretization error of LMC is
% \begin{align*}
%     W_2(\text{Law}(\bar{\bs{x}}_k), \text{Law}(\bs{x}_{kh})) =& \sqrt{d} \left(\sqrt{\frac{2}{2-h}\left(1 - (1-h)^{2k}\right)}  -  \sqrt{1 - e^{-2kh}} \right) \\
%     \stackrel{(i)}{\le}& \sqrt{d} \sqrt{1 - e^{-2kh}} \left( \sqrt{\frac{2}{2-h}} - 1\right) \\
%     \le& \frac{1}{4} \sqrt{d}h
% \end{align*}
% where $(i)$ is due to $1 - x < e^{-x}$. Therefore, our quantification of the global discretization error of LMC is already optimal.

Corollary \ref{corr:ic-msa} combined with the above result gives the following bound on the mixing time of LMC:
\begin{theorem}\label{thm:ic-lmc}(\textbf{Upper Bound of Mixing Time: LMC})
Suppose Assumption \ref{asp:smooth-and-strongly-convex}  and \ref{asp:linear_3rd_derivative} hold. If running LMC from $\bs{x}_0$, we then have 
\begin{equation*}
    \tau_{\mathrm{mix}}(\epsilon; W_2; \mathrm{LMC}) \leq \max\left\{4\kappa^2,  \frac{2C_\text{LMC}}{m} \frac{1}{\epsilon } \right\} \log \frac{2 W_2(\text{Law}(\bs{x}_0), \mu)}{\epsilon}
\end{equation*}
where $C_\text{LMC}$ is the same in Theorem \ref{thm:w2-lmc}. When high accuracy is needed, i.e., $\epsilon \le \frac{C_\text{LMC}}{2m\kappa^2}$, we have
\begin{equation*}
    \tau_{\mathrm{mix}}(\epsilon; W_2; \mathrm{LMC}) \leq  \frac{2C_\text{LMC}}{m} \frac{1}{\epsilon } \log \frac{2 W_2(\text{Law}(\bs{x}_0), \mu)}{\epsilon} = \tildeO{\frac{\sqrt{d}}{\epsilon}}.
\end{equation*}
\end{theorem}
% \begin{proof}
% The proof is straight by Corollary \ref{corr:ic-msa} and the computation of $C_{\text{LMC}}$ from Theorem \ref{thm:ic-lmc}. 
% \end{proof}

The $\tildeO{ \frac{\sqrt{d}}{\epsilon} }$ mixing time bound in $W_2$ distance improves upon the previous ones
\citep{dalalyan2017further, cheng2018convergence, durmus2019high, durmus2019analysis} in the dependence of $d$ and/or $\epsilon$. If further assuming $G = \mathcal{O}(L^2)$,  we then have $\Clmc = \mathcal{O}(\kappa^2 \sqrt{m} \sqrt{d})$ and Thm.\ref{thm:ic-lmc} shows the mixing time is $\tildeO{\frac{\kappa^2}{\sqrt{m}} \frac{\sqrt{d}}{\epsilon}}$, which also improves the $\kappa$ dependence in some previous results \citep{dalalyan2017further, cheng2018convergence} in the $m \leq 1$ regime. A brief summary is in Table \ref{tab:comparison}.

\begin{remark}[more comparison]
  The seminal work of \cite{li2019stochastic} provided mean-square analysis (their Thm.1) and obtained a $\tildeO{\frac{d}{\epsilon}}$ mixing time bound for LMC (their Example 1) under smoothness, strong convexity and Hessian Lipschitz conditions, consistent with that in \cite{durmus2019high}. By using our version (Thm.\ref{thm:global-error}) and tracking down constants' dimension-dependence, we are able to tighten it to $\tildeO{ \frac{\sqrt{d}}{\epsilon} }$. Worth clarifying is, dimension-dependence might not be the focus of \cite{li2019stochastic}; instead, it considered $\epsilon$-dependence and other discretizations, and showed for example that 1.5 SRK discretization has improved mixing time bound of $\tildeO{ \frac{d}{\epsilon^{2/3}}}$. The dimension dependence of this discretization, for example, can possibly be improved by our results too.
\end{remark}

% \li{TODO: discuss the dependence on condition number $\kappa$.}
% Due to the introduction of the constant $G$, $C_{\mathrm{LMC}}$ can not be cleanly expressed in condition number $\kappa$. However, under further %stringent
% assumptions such as $G = \mathcal{O}(L\sqrt{m}), mL = \Omega(1)$ and $\norm{\bs{x}_0}^2 = \mathcal{O}\left(\frac{d}{m}\right)$, it is possible to show that $C_{\mathrm{LMC}} = \mathcal{O}(\kappa \sqrt{d})$ and hence the dependence of the mixing time on condition number $\kappa$ can be made more explicit, i.e., $\tau_{\mathrm{mix}}(\epsilon; W_2; \mathrm{LMC}) = \tildeO{\frac{\kappa \sqrt{d}}{\epsilon}}$.

\setcounter{footnote}{0}

\begin{table}[t]
    \centering
    \begin{tabular}{|c|c|c|c|c|}
    \hline
     &  mixing time & Additional Assumption   \\ \hline
    \cite[Theorem 1]{dalalyan2017further}  & $\tildeO{\frac{\kappa^2}{m} \cdot \frac{d}{\epsilon^2}}$  & N/A  \\ \hline
    \cite[Theorem 1]{cheng2018convergence} &  $\tildeO{\frac{\kappa^2}{m} \cdot \frac{d}{\epsilon^2}}$ & N/A  \\ \hline
    \cite[Corollary 10]{durmus2019analysis} &   $\tildeO{ \frac{\kappa}{m} \cdot \frac{d}{\epsilon^2}}$ & N/A  \\ \hline
    \cite[Theorem 8]{durmus2019high} & $\tildeO{\frac{d}{\epsilon}}$ \tablefootnote{The dependence on $\kappa$ is not readily available from Theorem 8 in \cite{durmus2019high}.} 
    & $\norm{\nabla^2 f(\bs{x}) - \nabla^2 f(\bs{y})} \le \widetilde{L} \norm{\bs{x}- \bs{y}}$  \\ \hline
    This work (Theorem \ref{thm:ic-lmc}) & $\tildeO{ \frac{\kappa^2}{\sqrt{m}} \cdot \frac{\sqrt{d}}{\epsilon}}$ & Assumption \ref{asp:linear_3rd_derivative} and $G = \mathcal{O}(L^2)$ \tablefootnote{The $G = \mathcal{O}(L^2)$ assumption is only for $\kappa, m$ dependence. Removing it doesn't affect $d, \epsilon$ dependence.
    % Without it, one still obtains $\tildeO{\frac{\sqrt{d}}{\epsilon}}$ mixing time bound, i.e. does not affect $d$ and $\epsilon$ dependence.
    }
    \\ \hline
    \end{tabular}
    \caption{Comparison of mixing time results in 2-Wassertein distance of LMC with $L$-smooth and $m$-strongly-convex potential. Constant step size is used and accuracy tolerance $\epsilon$ is small enough.}\label{tab:comparison}
\end{table}

\paragraph{Optimality}
In fact, the $\tildeO{\frac{\sqrt{d}}{\epsilon}}$ mixing time of LMC has the optimal scaling one can expect. This is in terms of dependence on $d$ and $\epsilon$, over the class of all log-smooth and log-strongly-convex target measures. To illustrate this, consider the following Gaussian target distribution whose potential is
\begin{equation}\label{eq:quadratic}
    f(\bs{x}) = \frac{m}{2}\sum_{i=1}^d x_i^2 + \frac{L}{2}\sum_{i=d+1}^{2d} x_i^2,\quad \mbox{ with } m = 1, L \geq 4m.
\end{equation}
We now establish a lower bound on the mixing time of LMC algorithm for this target measure.

\begin{theorem}\label{thm:lower-bound}
(\textbf{Lower Bound of Mixing Time}) Suppose we run LMC for the target measure defined in Eq. \eqref{eq:quadratic} from $\bs{x}_0 = \bs{1}_{2d}$, then for any choice of step size $h > 0$ within stability limit, we have
\[
    \tau_{\mathrm{mix}}(\epsilon; W_2; \mathrm{LMC}) \geq \frac{\sqrt{d}}{8\epsilon} \log \frac{\sqrt{d}}{\epsilon} = \widetilde{\Omega}\left(\frac{\sqrt{d}}{\epsilon}\right).
\]
\end{theorem}

% \li{TODO: add limitation of our results.}
Combining Theorem \ref{thm:ic-lmc} and \ref{thm:lower-bound}, we see that mean-square analysis provides a tight bound for LMC and $\tildeO{\frac{\sqrt{d}}{\epsilon}}$ is the optimal scaling of LMC for target measures satisfying Assumptions \ref{asp:smooth-and-strongly-convex} and \ref{asp:linear_3rd_derivative}.

Note that the above optimality results only partially, but \textit{not} completely, close the gap between the upper and lower bounds of LMC over the entire family of log-smooth and log-strongly-convex target measures, because of one limitation of our result --- A(ssumption)\ref{asp:linear_3rd_derivative} is, despite of its close relation to the Hessian Lipschitz condition frequently used in the literature, still an extra condition. We tend to believe that A\ref{asp:linear_3rd_derivative} may not be essential, but rather an artifact of our proof technique. However, at this moment we cannot eliminate the possibility that the best scaling one can get out of Assumption \ref{asp:smooth-and-strongly-convex} only (no A\ref{asp:linear_3rd_derivative}) is worse than $\tildeO{\sqrt{d}/\epsilon}$. We'd like to further investigate this in future work.

\paragraph{Discussion} Besides \cite{li2019stochastic} (see the previous Remark), let's briefly discuss three more important sampling algorithms related to LMC. Two of them are Kinetic Langevin Monte Carlo (KLMC) and Randomized Midpoint Algorithm (RMA), both of which are discretizations of kinetic Langevin dynamics. The other is  Metropolis-Adjusted Langevin Algorithm (MALA) which uses the one-step update of LMC as a proposal and then accepts/rejects it with Metropolis-Hastings.

The $\tildeO{\frac{\sqrt{d}}{\epsilon}}$ mixing time in 2-Wasserstein distance of KLMC has been established for log-smooth and log-strongly-convex target measures in existing literature \citep{cheng2017underdamped, dalalyan2018sampling} and that was a milestone. Due to its better dimension dependence over previously best known results of LMC, KLMC is understood to be the analog of Nesterov's accelerated gradient method for sampling \citep{ma2019there}. Our findings show that LMC is able to achieve the same mixing time, although under an additional growth-at-infinity condition. However, this does not say anything about whether/how KLMC accelerates LMC, as the existing KLMC bound may still be not optimal. We also note KLMC has better condition number dependence than our current LMC result, although the $\kappa$ dependence in our bound may not be tight. %\tao{Ruilin, I added several sentences here. Do they sound okay to you?}\li{Sounds good to me. One minor comment is the usage of the word 'verdict', I looked up in a few dictionaries but did not see it being used as a verb.}

% Despite having the same theoretical mixing time bounds, KLMC can be practically advantageous to LMC in certain cases, for example, when smoothness coefficient $L$ is large, the largest step size allowed by Theorem \ref{thm:w2-lmc} is $\Theta\left(\frac{1}{\kappa L}\right)$, whereas KLMC can make use of larger $\Theta\left(\frac{1}{\kappa \sqrt{L} }\right)$ step size \citep[Theorem 2]{dalalyan2018sampling}.

RMA \citep{shen2019randomized} is based on a brilliant randomized discretization of kinetic Langevin dynamics and shown to have further improved dimension dependence (and other pleasant properties). From the perspective of this work, we think it is because RMA is able to break the strong order barrier due to the randomization, and more investigations based on mean-square analysis should be possible.

For MALA, a recent breakthrough \citep{chewi2020optimal} establishes a $\tildeO{\sqrt{d}}$ mixing time in $W_2$ distance with warm start, and the dimension dependence is shown to be optimal. We see that without the Metropolis adjustment, LMC (under additional assumptions such as A\ref{asp:linear_3rd_derivative}) can also achieve the same dimension dependence as MALA. But unlike LMC, MALA only has logarithmic dependence on $\frac{1}{\epsilon}$. With warm-start, is it possible/how to improve the dependence of $\frac{1}{\epsilon}$ for LMC, from polynomial to logarithmic? This question is beyond the scope of this paper but worth further investigation.

\section{Numerical Examples}\label{sec:numerical}
\vspace{-6pt}
This section numerically verifies our theoretical findings for LMC in Section \ref{sec:LMC}, with a particular focus on the dependence of the discretization error in Theorem \ref{thm:w2-lmc} on dimension $d$ and step size $h$. To this end, we consider two target measures specified by the following two potentials:
\begin{equation}\label{eq:test-potentials}
    f_1(\bs{x}) = \frac{1}{2} \norm{\bs{x}}^2 + \log\left(\sum_{i=1}^d e^{x_i} \right) \quad \mbox{and} \quad f_2(\bs{x}) = \frac{1}{2} \norm{\bs{x}}^2 - \frac{1}{2d^\frac{1}{2}} \sum_{i=1}^d \cos\left(d^\frac{1}{4} x_i \right).
\end{equation}
It is not hard to see $f_1$ is 2-smooth and 1-strongly convex, $f_2$ is $\frac{3}{2}$-smooth and 1-strongly-convex, and both satisfy Assumption \ref{asp:linear_3rd_derivative}. $f_2$ is also used in \citep{chewi2020optimal} to illustrate the optimal dimension dependence of MALA. Explicit expression of 2-Wasserstein distance between non-Gaussian distributions is typically not available, instead, we use the Euclidean norm of the mean error as a surrogate because $\norm{\mathbb{E} \bar{\bs{x}}_k - \mathbb{E}_\mu \bs{x}} \leq W_2(\mathrm{Law}(\bar{\bs{x}}_k), \mu)$ due to Jensen's inequality. %\tao{how is Thm.\ref{thm:w2-lmc} relevant?}\li{I rewrote it and removed Thm. \ref{thm:w2-lmc}}
% \begin{equation}\label{eq:numerical-w2}
%     \norm{\mathbb{E} \bar{\bs{x}}_k - \mathbb{E}_\mu \bs{x}} \leq W_2(\mathrm{Law}(\bar{\bs{x}}_k), \mu) \leq \sqrt{2} e^{- mkh}  W_2(\text{Law}(\bs{x}_0), \mu) + \underbrace{\sqrt{2}C_{\text{LMC}} h}_{\text{discretization error}}.
% \end{equation}
To obtain an accurate estimate of the ground truth, we run $10^8$ independent LMC realizations using a tiny step size (h = 0.001), each till a fixed, long enough time, and use the empirical average to 
% We parallel the jobs to a cluster for faster processing. 
approximate $\mathbb{E}_\mu \bs{x}$.

To study the dimension dependence of sampling error,
%\tao{did you mean `sampling error' instead of `discretization error' (note the typo too)?}\li{When running LMC for long enough, sampling error is dominated by discretization error, so I thought they can be used interchangeably. But I agree that discretization error is just an intermediate quantity and we should use the term sampling error here.}
we fix step size $h=0.1$, and for each $d \in \left\{ 1, 2, 5, 10, 20, 50, 100, 200, 500, 1000\right\}$, we simulate $10^4$ independent Markov chains using LMC algorithm for $100$ iterations, which is long enough for the chain to be well-mixed. %\tao{is this supposed to be long enough time?}\li{Yes, the $0.1$ step size would give a discretization error of order $10^{-1}$. For $K=100$, its corresponding continuous time is $T=10$ and error due to continuous dynamics is $e^{-10} \approx 4\times 10^{-5}$. So discretization error dominates and the chain is well-cocnverged in this sense. }. 
The mean and the standard deviation of the sampling error corresponding to the last 10 iterates are recorded.

To study step size dependence of sampling %\tao{sampling or discretization}\li{updated.}
error, we fix $d=10$ and experiment with  step size $h \in \left\{ 1,2,3,4,5,6,7,8,9,10 \right\} \times 10^{-1}$. We run LMC till $T=20$, i.e., $\ceil{\frac{T}{h}}$ iterations for each $h$. The procedure is repeated $10^4$ times with different random seeds to obtain independent samples. When the corresponding continuous time $t=kh > 10$, we see from Eq. \eqref{eq:non-asymptotic-bound-lmc} that LMC is well converged and the sampling error is saturated by the discretization error. Therefore, for each $h$, we take the last $\ceil{\frac{10}{h}}$ iterates and record the mean and standard deviation of their sampling error.

\begin{figure}[ht]
    \centering
    \begin{subfigure}{0.24\textwidth}
		\centering
		\includegraphics[width=\textwidth]{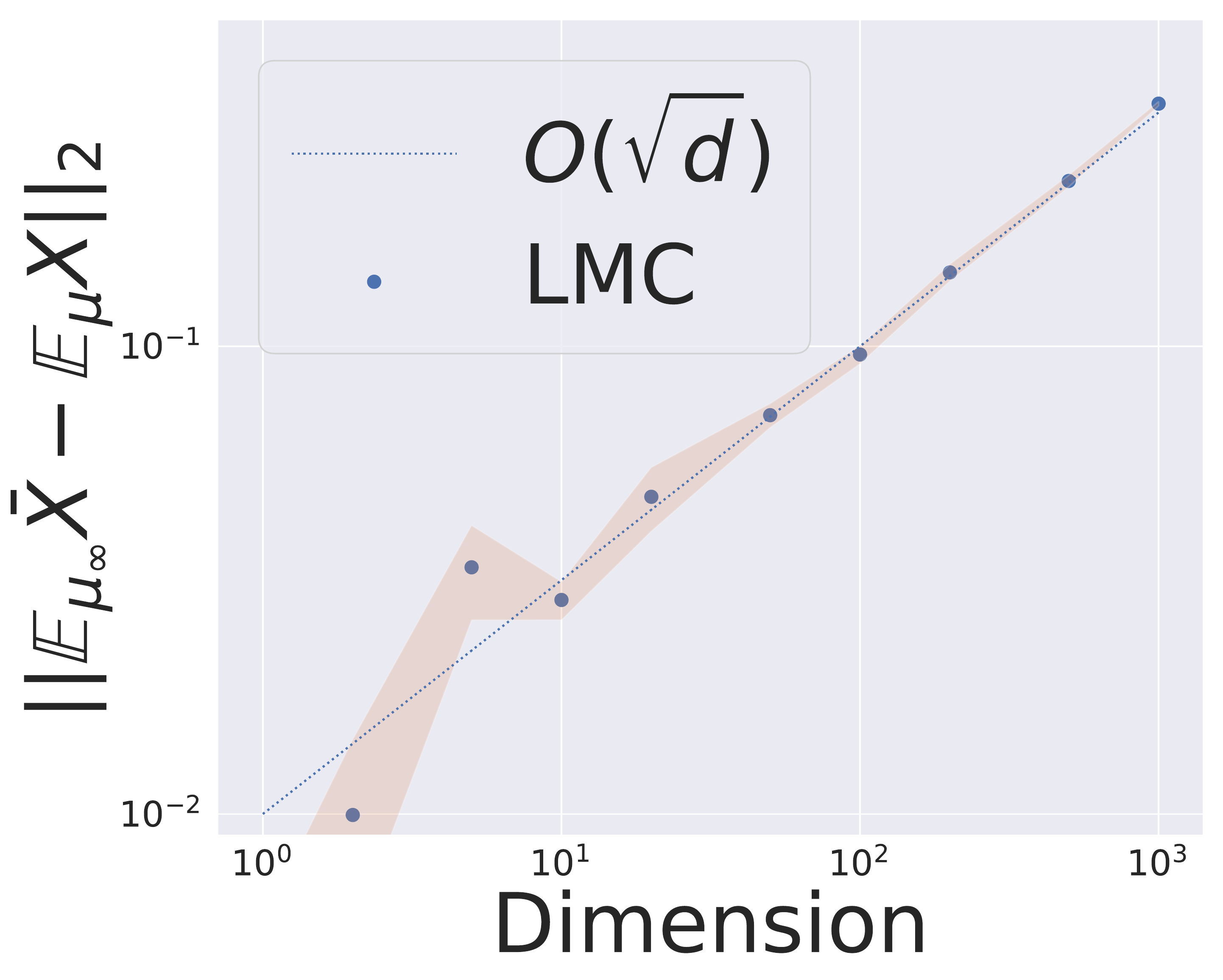}
		\vspace{-15pt}\caption{$f_1$: $d$ dependence} \label{fig:f1_d}	
	\end{subfigure}
    \begin{subfigure}{0.24\textwidth}
		\centering
		\includegraphics[width=\textwidth]{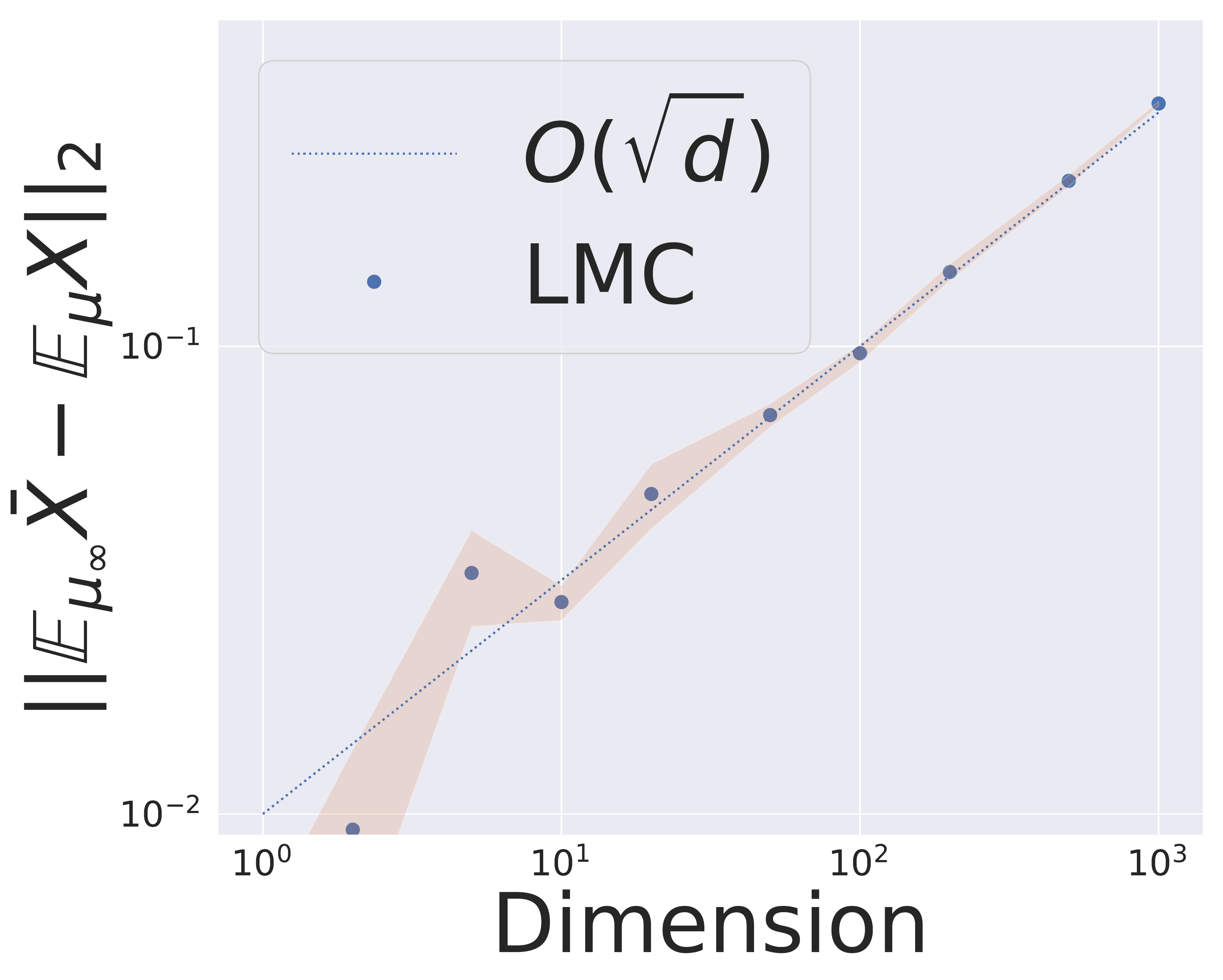}
		\vspace{-15pt}\caption{$f_2$: $d$ dependence %\tao{is it supposed to be (almost) identical to the $f_1$ result?} \li{Yes, $f_1$ and $f_2$ have very very close numerical values in this experiment, and that is why they look almost the same.}
		} \label{fig:f2_d}
	\end{subfigure}
    \begin{subfigure}{0.24\textwidth}
		\centering
		\includegraphics[width=\textwidth]{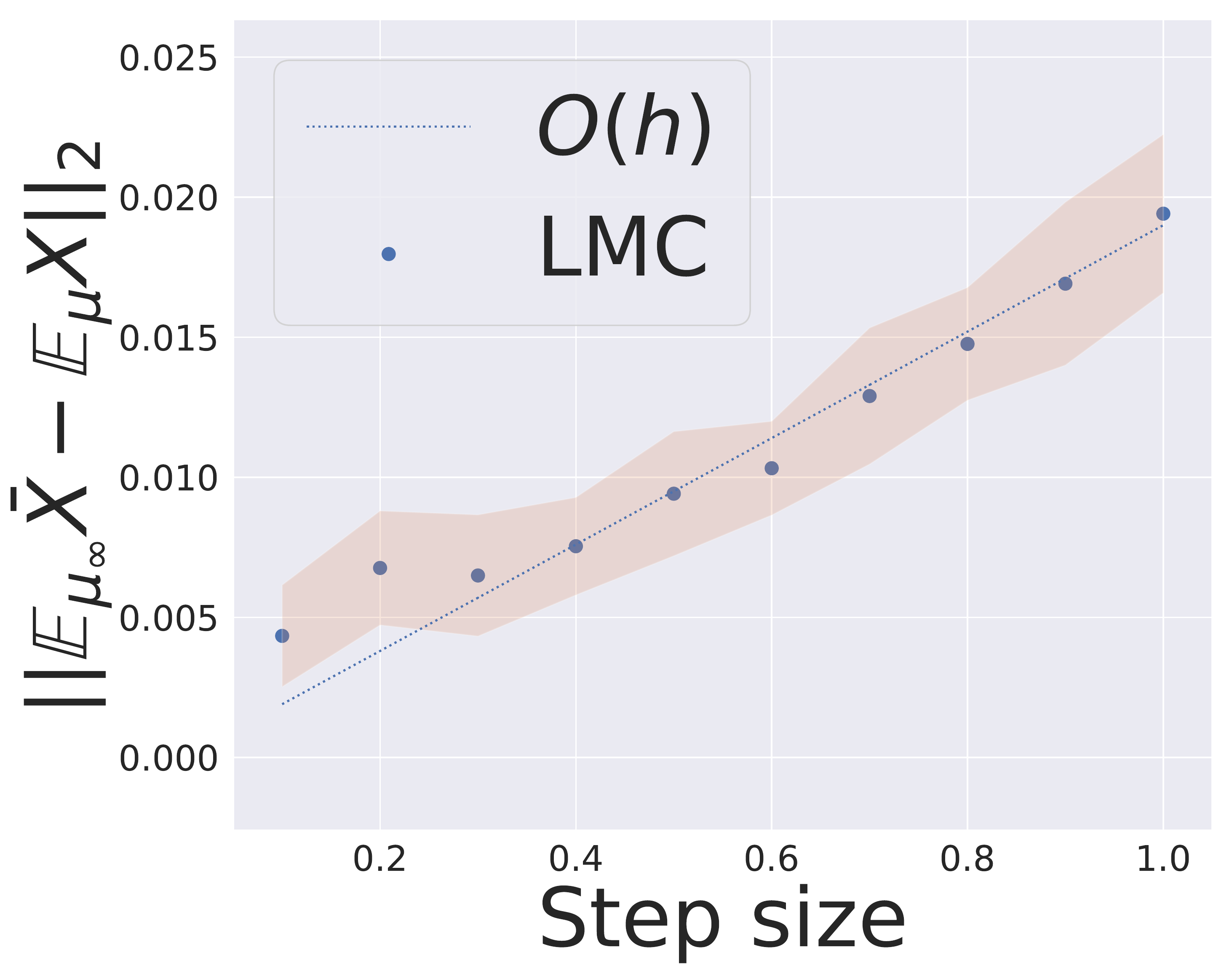}
		\vspace{-15pt}\caption{$f_1$: $h$ dependence} \label{fig:f1_h}	
	\end{subfigure}
    \begin{subfigure}{0.24\textwidth}
		\centering
		\includegraphics[width=\textwidth]{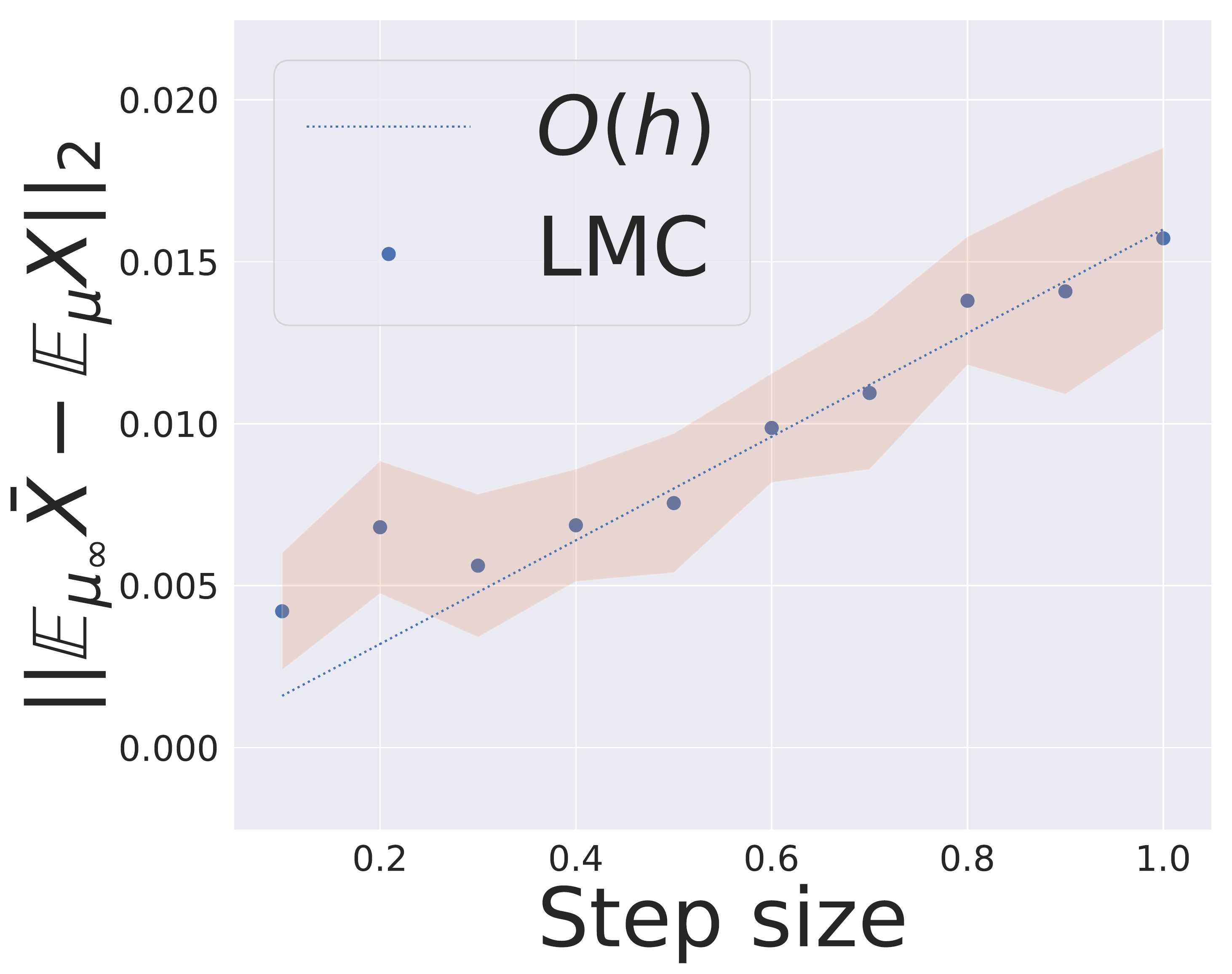}
		\vspace{-15pt}\caption{$f_2$: $h$ dependence} \label{fig:f2_h}
	\end{subfigure}
	\vspace{-8pt}
    \caption{Dependence of the sampling error of LMC on dimension $d$ and step size $h$ for $f_1$ and $f_2$. Both axes in Fig.\ref{fig:f1_d} \& \ref{fig:f2_d} are in log scale. Shaded areas in Fig.\ref{fig:f1_d} \& \ref{fig:f2_d} represent one std. of the last 10 iterations. Shaded areas in Fig.\ref{fig:f1_h} \& \ref{fig:f2_h} represent one std. of the last $\ceil{\frac{10}{h}}$ iterations.
    }\label{fig:numerical}
\end{figure}

Results shown in Fig.\ref{fig:numerical} are consistent with our theoretical analysis of the sampling error. % \tao{I think it is better that we talk about sampling error instead of discretization error; the latter is just an intermediate step required by our tool}\li{updated}. 
Both linear dependence on $\sqrt{d}$ and $h$ can be supported by the empirical evidence. Note results with smaller $h$ are less accurate because one starts to see the error of empirical approximation due to finite samples. %Experiments were conducted on a machine with a 2.20GHz Intel(R) Xeon(R) E5-2630 v4 CPU and an Nvidia GeForce GTX 1080 GPU. 

\section{Conclusion}\label{sec:conclusion}
\vspace{-6pt}
Via a refined mean-square analysis of Langevin Monte Carlo algorithm, we obtain an improved and optimal $\tildeO{ \nicefrac{\sqrt{d}}{\epsilon} }$ bound on its mixing time, %\tao{i noticed somewhere you used tildeO, and widetilde O elsewhere; any difference?}\li{No difference, tildeO is just a macro I defined for widetilde O for convenience.} \tao{table 1 for example has $\widetilde{O}$ instead of $\tilde{\mathcal{O}}$}\li{I have updated the draft, replacing all $\widetilde{O}()$ to $\tildeO{}$}
which was previously thought to be obtainable only with the addition of momentum. This was under the standard smoothness and strongly-convexity assumption, plus an addition linear growth condition on the third-order derivative of the potential function, similar to Hessian Lipschitz condition already popularized in the literature.

Here are some possible directions worth further investigations. (i) %In data-intensive applications, stochastic gradients are typically used for better scalability. It seems natural to 
Combine mean-square analysis with stochastic gradient analysis to study SDE-based stochastic gradient MCMC methods;  (ii) Is it still possible to obtain $\sqrt{d}$-dependence without A\ref{asp:linear_3rd_derivative}, i.e., only under log-smooth and log-strongly-convex conditions? (iii) Applications of mean-square analysis to other SDEs and/or discretizations; (iv) Motivated by \citet{chewi2020optimal}, it would be interesting to know whether the dependence on $\frac{1}{\epsilon}$ can be improved to logarithmic, for example if LMC is initialized at a warm start.

\section*{Acknowledgments}
The authors thank Andre Wibisono, Chenxu Pang, anonymous reviewers and area chair for suggestions that significantly improved the quality of this paper. MT was partially supported by NSF DMS-1847802 and ECCS-1936776. This work was initiated when HZ was a professor and RL was a PhD student at Georgia Institute of Technology.

\bibliography{reference}
\bibliographystyle{iclr2022_conference}

\newpage
\appendix
\section{Proof of Results in Section \ref{sec:MSA}}
\subsection{Proof of Theorem \ref{thm:global-error} (Global Integration Error, Infinite Time Version)}

\begin{proof}
We write the solution of an SDE by $\bs{x}_{t_0, \bs{x}_{t_0}}(t_0 + t)$ when the dependence on initialization needs highlight. Denote $t_k = kh$ and $\bs{x}_{t_k} = \bs{x}_k$ for better readability.

We will first make an easy observation that contraction and bounded 2nd-moment of the invariant distribution lead to bounded 2nd-moment of the SDE solution for all time: let $\bs{y}_0$ be a random variable following the invariant distribution of Eq. \eqref{eq:sde}, i.e., $\bs{y}_0 \sim \mu$, then $\bs{y}_t \sim \mu$ and 
\begin{align*}
    \mathbb{E}\norm{\bs{x}_t}^2 \le& 2\mathbb{E}\norm{\bs{x}_t - \bs{y}_t}^2 + 2\mathbb{E}\norm{\bs{y}_t}^2 \\
    \le& 2\mathbb{E} \norm{\bs{x}_0 - \bs{y}_0}^2 \exp(-2\beta t) + 2\mathbb{E}\norm{\bs{y}_t}^2 \\
    \le& 4\mathbb{E} (\norm{\bs{x}_0}^2 + \norm{\bs{y}_0}^2) \exp(-2\beta t) + 2\mathbb{E}\norm{\bs{y}_t}^2 \\
    =& 4 \mathbb{E}\norm{\bs{x}_0}^2 \exp(-2\beta t) + \left(2 + 4\exp(-2\beta t)\right) \mathbb{E}_{\bs{y} \sim \mu} \norm{\bs{y}}^2 \\
    \le& 4 \mathbb{E}\norm{\bs{x}_0}^2 + 6 \int_{\mathbb{R}^d} \norm{\bs{y}}^2 d\mu \triangleq U^2
\end{align*}
%\tao{Ruilin: what is $\bs{w}$? how did 2nd line->3rd line work? why isn't 2nd term in the last line 7*integral but 5*integral?}\li{$\bs{w}$ is just an integration variable, I have changed it to $\bs{y}$ for clarity. There seemed to be some issues with the coefficient, I have also updated them.}
and then it follows that
\begin{equation}\label{eq:bound_on_sampler}
    \mathbb{E}\norm{\bar{\bs{x}}_k}^2 \leq 2\mathbb{E}\norm{\bar{\bs{x}}_k - \bs{x}_k}^2 + 2 \mathbb{E}\norm{\bs{x}_k}^2 \leq 2e_k^2 + 2U^2.
\end{equation}

Denote $\innerprodA{\bs{x}}{\bs{y}} = \innerprod{A\bs{x}}{A\bs{y}}$, $\norm{\bs{x}}_A = \norm{A\bs{x}}$ and 
\begin{equation}
    f_k = \left\{\mathbb{E}\norm{\bs{x}_k - \bar{\bs{x}}_k}_A^2 \right\}^\frac{1}{2}
\end{equation}
where $A$ is the non-singular matrix from Equation \eqref{eq:contraction_def}. Also denote that largest and smallest singular values of $A$ by $\sigma_\mathrm{max}$ and $\sigma_\mathrm{min}$, respectively, and the condition number of $A$ by $\kappa_A = \frac{\sigma_\mathrm{max}}{\sigma_\mathrm{min}}$. Recall $e_k = \mathbb{E}\norm{\bs{x}_k - \bar{\bs{x}}_k} $, it is easy to see that
\begin{equation}
    \sigma_\mathrm{min} e_k \leq f_k \leq \sigma_\mathrm{max} e_k.
\end{equation}

Further, we have the following decomposition
\begin{align}
    f_{k+1}^2 
    =& \mathbb{E} \norm{\bs{x}_{k+1} - \bar{\bs{x}}_{k+1}}_A^2 \nonumber\\
    =& \mathbb{E} \norm{\bs{x}_{t_k, \bs{x}_{t_k}}(t_{k+1}) - \bs{x}_{t_k, \bar{\bs{x}}_k}(t_{k+1}) + \bs{x}_{t_k, \bar{\bs{x}}_k}(t_{k+1}) - \bar{\bs{x}}_{k+1}}_A^2 \nonumber\\
    =&  \underbrace{\mathbb{E} \norm{\bs{x}_{t_k, \bs{x}_{t_k}}(t_{k+1}) - \bs{x}_{t_k, \bar{\bs{x}}_k}(t_{k+1})}_A^2}_{\circled{1}} 
        + \underbrace{\mathbb{E} \norm{\bs{x}_{t_k, \bar{\bs{x}}_k}(t_{k+1}) - \bar{\bs{x}}_{k+1}}_A^2}_{\circled{2}} \label{eq:decomposition}\\
        &+ 2\underbrace{\mathbb{E} \innerprod{ A\left(\bs{x}_{t_k, \bs{x}_{t_k}}(t_{k+1}) - \bs{x}_{t_k, \bar{\bs{x}}_k}(t_{k+1})\right)}{A\left(\bs{x}_{t_k, \bar{\bs{x}}_k}(t_{k+1}) - \bar{\bs{x}}_{k+1}\right) }}_{\circled{3}}. \nonumber
\end{align}

Term $\circled{1}$ is taken care of the contraction property
\begin{equation}\label{eq:term1_msa}
    \mathbb{E} \norm{\bs{x}_{t_k, \bs{x}_{t_k}}(t_{k+1}) - \bs{x}_{t_k, \bar{\bs{x}}_k}(t_{k+1})}_A^2 \leq f_k^2 \exp(-2\beta h).
\end{equation}

Term $\circled{2}$ is dealt with by the bound on local strong error
\begin{equation}\label{eq:term2_msa}
    \mathbb{E} \norm{\bs{x}_{t_k, \bar{\bs{x}}_k}(t_{k+1}) - \bar{\bs{x}}_{k+1}}_A^2 \leq \sigma_\mathrm{max}^2 \left( C_2^2 + D_2^2 \mathbb{E}\norm{\bar{\bs{x}}_k}^2 \right) h^{2p_2}.
\end{equation}
% \tao{norm on the RHS missing subscript A?}\li{Here I used $\norm{\cdot}_A \leq \sigma_{\mathrm{max}} \norm{\cdot}$, so we are not missing the subscript. The same for the two comments below. Basically, every transition from $\norm{\cdot}_A$ to $\norm{\cdot}$ will introduce $\sigma_\mathrm{max}$.}\tao{I see. Thanks. Feel free to remove my next two comments if they are incorrect.}

Term $\circled{3}$ requires more efforts to cope with, and by the decomposition in Eq. \eqref{eq:z} we have
\begin{align}
    & \mathbb{E} \innerprodA{(\bs{x}_{t_k, \bs{x}_{t_k}}(t_{k+1}) - \bs{x}_{t_k, \bar{\bs{x}}_k}(t_{k+1})}{\bs{x}_{t_k, \bar{\bs{x}}_k}(t_{k+1}) - \bar{\bs{x}}_{k+1}} \nonumber\\
    =& \mathbb{E} \innerprodA{\bs{x}_k - \bar{\bs{x}}_k}{\bs{x}_{t_k, \bar{\bs{x}}_k}(t_{k+1}) - \bar{\bs{x}}_{k+1}} + \mathbb{E} \innerprodA{\bs{z}_h(\bs{x}_k, \bar{\bs{x}}_k)}{\bs{x}_{t_k, \bar{\bs{x}}_k}(t_{k+1}) - \bar{\bs{x}}_{k+1}} \nonumber\\
    \stackrel{(i)}{=}& \mathbb{E} \innerprodA{\bs{x}_k - \bar{\bs{x}}_k}{ \mathbb{E}  [\bs{x}_{t_k,\bar{\bs{x}}_k}(t_{k+1}) - \bar{\bs{x}}_{k+1} | \mathcal{F}_k] } + \mathbb{E} \innerprodA{\bs{z}_h(\bs{x}_k, \bar{\bs{x}}_k)}{\bs{x}_{t_k, \bar{\bs{x}}_k}(t_{k+1}) - \bar{\bs{x}}_{k+1}} \nonumber\\
    \stackrel{(ii)}{\le}& f_k \, \left(\mathbb{E} \norm{\mathbb{E}  [\bs{x}_{t_k,\bar{\bs{x}}_k}(t_{k+1}) - \bar{\bs{x}}_{k+1} | \mathcal{F}_k]}_A^2\right)^\frac{1}{2}  +  \left(\mathbb{E} \norm{\bs{z}_h(\bs{x}_k, \bar{\bs{x}}_k)}_A^2\right)^\frac{1}{2} \,  \left(\mathbb{E} \norm{\bs{x}_{t_k, \bar{\bs{x}}_k}(t_{k+1}) - \bar{\bs{x}}_{k+1}}_A^2\right)^\frac{1}{2} \nonumber\\
    \stackrel{(iii)}{\leq} & \sigma_\mathrm{max} f_k \, \left(\mathbb{E} \norm{\mathbb{E}  [\bs{x}_{t_k,\bar{\bs{x}}_k}(t_{k+1}) - \bar{\bs{x}}_{k+1} | \mathcal{F}_k]}^2\right)^\frac{1}{2}  + \sigma_\mathrm{max}^2 \left(\mathbb{E} \norm{\bs{z}_h(\bs{x}_k, \bar{\bs{x}}_k)}^2\right)^\frac{1}{2} \,  \left(\mathbb{E} \norm{\bs{x}_{t_k, \bar{\bs{x}}_k}(t_{k+1}) - \bar{\bs{x}}_{k+1}}^2\right)^\frac{1}{2} \nonumber\\
    \stackrel{(iv)}{\le}& \sigma_\mathrm{max} f_k \, \left( C_1 + D_1 \sqrt{\mathbb{E} \norm{\bar{\bs{x}}_k}^2}  \right) h^{p_1} +  
    \kappa_A \sigma_\mathrm{max} C_0 f_k \sqrt{h} \, \left( C_2 + D_2 \sqrt{\mathbb{E} \norm{\bar{\bs{x}}_k}^2} \right) h^{p_2} \nonumber\\
    \stackrel{(v)}{\le}& \kappa_A \sigma_\mathrm{max} (C_1 + C_0 C_2) e_k h^{p_2 + \frac{1}{2}} + \kappa_A \sigma_\mathrm{max} (D_1 + C_0 D_2) \sqrt{\mathbb{E} \norm{\bar{\bs{x}}_k}^2} f_k h^{p_2 + \frac{1}{2}}  \label{eq:term3_msa}
    % \stackrel{(v)}{\le}& \frac{\beta}{8}e_k^2 h + \frac{2\left(C_1 + C_0 C_2\right)^2}{\beta} h^{2p_2} + (D_1 + C_0 D_2) \sqrt{\mathbb{E} \norm{\bar{\bs{x}}_k}^2} e_k h^{p_2 + \frac{1}{2}}
\end{align}
% \tao{norms in last 3 lines missing subscript A?}
where $(i)$ uses the tower property of conditional expectation and $\mathcal{F}_k$ is the filtration at $k$-th iteration, $(ii)$ uses Cauchy-Schwarz inequality, $(iii)$ is due to the relationship between $e_k$ and $f_k$, $(iv)$ is due to local weak error, local strong error and Eq. \eqref{eq:z}, and $(v)$ is due to $p_1 \ge p_2 + \frac{1}{2}$ and $0 < h \leq h_0 \leq 1$.

% \tao{More norms below missing A? One thing we can do is to say, after the definition under eq.\eqref{eq:bound_on_sampler}, that all inner products and all norms will be w.r.t. $A$ unless otherwise specified.}

Now plug Eq. \eqref{eq:term1_msa}, \eqref{eq:term2_msa} and \eqref{eq:term3_msa} in Eq. \eqref{eq:decomposition}, we obtain
\begin{align*}
    f_{k+1}^2 \le& f_k^2 \exp(-2\beta h) + \sigma_\mathrm{max}^2 \left( C_2^2 + D_2^2 \mathbb{E}\norm{\bar{\bs{x}}_k}^2 \right) h^{2p_2} + \kappa_A \sigma_\mathrm{max}(C_1 + C_0 C_2) f_k h^{p_2 + \frac{1}{2}}  \\
    &+ \kappa_A \sigma_\mathrm{max} (D_1 + C_0 D_2) \sqrt{\mathbb{E} \norm{\bar{\bs{x}}_k}^2} f_k h^{p_2 + \frac{1}{2}} \\
    \stackrel{(i)}{\le}&  \left( 1 - \frac{7}{8}\beta h \right)f_k^2  + \sigma_\mathrm{max}^2 \left( C_2^2 + D_2^2 \mathbb{E}\norm{\bar{\bs{x}}_k}^2 \right) h^{2p_2} + \kappa_A \sigma_\mathrm{max}(C_1 + C_0 C_2) f_k h^{p_2 + \frac{1}{2}}  \\
    &+ \kappa_A \sigma_\mathrm{max} (D_1 + C_0 D_2) \sqrt{\mathbb{E} \norm{\bar{\bs{x}}_k}^2} f_k h^{p_2 + \frac{1}{2}} \\
    \stackrel{(ii)}{\leq} & \left( 1 - \frac{7}{8}\beta h \right)f_k^2  + \kappa_A \sigma_\mathrm{max} \left( C_1 + C_0 C_2 + \sqrt{2}U(D_1 + C_0 D_2) \right) f_k h^{p_2 + \frac{1}{2}} + 2\kappa_A^2 D_2^2 f_k^2 h^{2p_2} \\
    &+ \sqrt{2} \kappa_A^2 (D_1 + C_0 D_2) f_k^2 h^{p_2 + \frac{1}{2}} + \sigma_\mathrm{max}^2 \left( C_2^2 + 2D_2^2U^2 \right) h^{2p_2}  \\
    \stackrel{(iii)}{\leq}& \left( 1 - \frac{7}{8}\beta h \right)f_k^2  + \kappa_A \sigma_\mathrm{max}  \left( C_1 + C_0 C_2 + \sqrt{2}U(D_1 + C_0 D_2) \right) f_k h^{p_2 + \frac{1}{2}} + \frac{3\beta}{8} f_k^2 h \\
    &+ \sigma_\mathrm{max}^2 \left( C_2^2 + 2D_2^2U^2 \right) h^{2p_2} \\
    =& \left( 1 - \frac{1}{2}\beta h \right)f_k^2  + \kappa_A \sigma_\mathrm{max}\left( C_1 + C_0 C_2 + \sqrt{2}U(D_1 + C_0 D_2) \right) f_k h^{p_2 + \frac{1}{2}} \\
    &+ \sigma_\mathrm{max}^2 \left( C_2^2 + 2D_2^2U^2 \right) h^{2p_2} \\
    \stackrel{(iv)}{\leq} & \left( 1 - \frac{1}{2}\beta h \right)f_k^2 + \frac{\beta}{4}f_k^2 h + \frac{\kappa_A^2 \sigma_\mathrm{max}^2 \left( C_1 + C_0 C_2 + \sqrt{2}U(D_1 + C_0 D_2) \right)^2}{\beta} h^{2p_2} \\
    &+ \sigma_\mathrm{max}^2 \left( C_2^2 + 2D_2^2U^2 \right) h^{2p_2} \\
    =& \left( 1 - \frac{1}{4}\beta h \right)f_k^2 + \kappa_A^2 \sigma_\mathrm{max}^2 \left( \frac{\left( C_1 + C_0 C_2 + \sqrt{2}U(D_1 + C_0 D_2) \right)^2}{\beta} +  C_2^2 + 2D_2^2U^2 \right) h^{2p_2}
\end{align*}
where $(i)$ is due to the assumption $0 < h \leq\frac{1}{4\beta}$ and  $e^{-x} \leq1 - x + \frac{x^2}{2} \mbox{ for } 0 < x < 1$, $(ii)$ is due to the upper bound on $\mathbb{E}\norm{\bar{\bs{x}}_k}^2$ in Eq. \eqref{eq:bound_on_sampler}, $(iii)$ holds provided when $h \leq \min\left\{  \left(\frac{\sqrt{\beta}}{4\sqrt{2} \kappa_A D_2} \right)^\frac{1}{p_2 - \frac{1}{2}}, \left( \frac{\beta}{8\sqrt{2}\kappa_A^2 (D_1 + C_0 D_2)} \right)^\frac{1}{p_2 - \frac{1}{2}}   \right\}$ and $(iv)$ is due to Cauchy-Schwarz inequality.

Unfolding the above inequality gives us
\begin{align*}
    f_{k+1}^2 
    \le& \frac{4}{\beta} \kappa_A^2 \sigma_\mathrm{max}^2 \left( \frac{\left( C_1 + C_0 C_2 + \sqrt{2}U(D_1 + C_0 D_2) \right)^2}{\beta} +  C_2^2 + 2D_2^2U^2 \right) h^{2p_2-1}.
\end{align*}
Taking square root on both sides and using $\sqrt{a^2 + b^2 + c^2} \leq a + b + c, \forall a, b, c \geq 0$ yields
\begin{align*}
    f_{k+1} \leq \frac{2}{\sqrt{\beta}} \kappa_A \sigma_\mathrm{max} \left( \frac{C_1 + C_0 C_2 + \sqrt{2}U(D_1 + C_0 D_2) }{\sqrt{\beta}} +  C_2 + \sqrt{2} D_2U \right) h^{p_2 - \frac{1}{2}}.
\end{align*}
Finally using the relationship between $e_k$ and $f_k$, we obtain
\[
    e_k \leq \frac{2}{\sqrt{\beta}} \kappa_A^2 \left( \frac{C_1 + C_0 C_2 + \sqrt{2}U(D_1 + C_0 D_2) }{\sqrt{\beta}} +  C_2 + \sqrt{2} D_2U \right) h^{p_2 - \frac{1}{2}}.
\]
\end{proof}

\subsection{Proof of Theorem \ref{thm:w2-non-asymptotic} (Non-Asymptotic Sampling Error Bound: General Case)}
\begin{proof}
Let $\bs{y}_0 \sim \mu$ and $(\bs{x}_0, \bs{y}_0)$ are coupled such that $\mathbb{E}\norm{\bs{x}_0 - \bs{y}_0}^2 = W_2^2(\text{Law}(\bs{x}_0), \mu)$. Denote the solution of Eq. \eqref{eq:sde} starting from $\bs{x}_0, \bs{y}_0$ by $\bs{x}_t, \bs{y}_t$ respectively, and $t_k = kh$. We have
\begin{align*}
    W_2(\text{Law}(\bar{\bs{x}}_k), \mu) \le& W_2(\text{Law}(\bar{\bs{x}}_k), \text{Law}(\bs{x}_{t_k})) + W_2(\text{Law}(\bs{x}_{t_k}), \mu)\\
    \le& \sqrt{\mathbb{E}\norm{\bar{\bs{x}}_k - \bs{x}_{t_k}}^2} + \sqrt{\mathbb{E}\norm{\bs{x}_{t_k} - \bs{y}_{t_k}}^2} \\
    \stackrel{(i)}{\le}& e_k + \sqrt{\mathbb{E}\norm{\bs{x}_0 - \bs{y}_0}^2 \exp\left(-2\beta t_k\right)} \\
    =& e_k + \exp\left(-\beta t_k\right) W_2(\text{Law}(\bs{x}_0), \mu)
\end{align*}
where $(i)$ is due to the contraction assumption on Eq. \eqref{eq:sde}. Invoking the conclusion of Theorem \ref{thm:global-error} completes the proof.
\end{proof}

\subsection{Proof of Corollary \ref{corr:ic-msa} (Upper Bound of Mixing Time: General Case)}
\begin{proof}
Given any tolerance $\epsilon > 0$,  we know from Theorem \ref{thm:w2-non-asymptotic} that if $k$ is large enough and $h$ is small enough such that
\begin{align}
    & \exp\left( - \beta kh\right)  W_2(\text{Law}(\bs{x}_0), \mu)  \leq\frac{\epsilon}{2} \label{eq:requirement-1}.\\
    & C h^{p_2 - \frac{1}{2}} \leq\frac{\epsilon}{2} \label{eq:requirement-2}
\end{align}
we then have $W_2(\text{Law}(\bar{\bs{x}}_k), \mu) \leq\epsilon$. Solving Inequality \eqref{eq:requirement-1} yields 
\begin{equation}\label{eq:lower-bound-on-k}
    k \ge \frac{1}{\beta h} \log \frac{2 W_2(\text{Law}(\bs{x}_0), \mu)}{\epsilon} \triangleq   k^\star
\end{equation}
To minimize the lower bound, we want pick step size $h$ as large as possible. Besides $h \leq h_1$, Eq. \eqref{eq:requirement-2} poses further constraint on $h$, hence we have 
\[
    h \leq\min \left\{h_1, \left(\frac{\epsilon}{2C}\right)^{\frac{1}{p_2 - \frac{1}{2}}}\right\}.
\]
Plug the upper bound of $h$ in Eq. \eqref{eq:lower-bound-on-k}, we have
\[
    k^\star = \max\left\{\frac{1}{\beta h_1},  \frac{1}{\beta} \left(\frac{2C}{\epsilon}\right)^{\frac{1}{p_2 - \frac{1}{2}}} \right\} \log \frac{2 W_2(\text{Law}(\bs{x}_0), \mu)}{\epsilon}.
\]
When high accuracy is needed, i.e., $\epsilon < 2C h_1^{p_2 - \frac{1}{2}}$, we have 
\[
    k^\star =  \frac{(2C)^{\frac{1}{p_2 - \frac{1}{2}}}}{\beta} \frac{1}{\epsilon^{\frac{1}{p_2 - \frac{1}{2}}} } \log \frac{2 W_2(\text{Law}(\bs{x}_0), \mu)}{\epsilon} = \tildeO{ \frac{C^{\frac{1}{p_2 - \frac{1}{2}}}}{\beta} \, \frac{1}{\epsilon^{\frac{1}{p_2 - \frac{1}{2}}} }}.
\]
\end{proof}

\section{Proof of Results in Section \ref{sec:LMC}} 
\subsection{Proof of Theorem \ref{thm:w2-lmc} (Non-Asymptotic Error Bound: LMC)}
\begin{proof}
From Lemma \ref{lemma:contraction} we know that Langevin dynamics is a member of the family of contractive SDE, and with a contraction rate of strong-convexity coefficient $\beta = m$ (w.r.t. identity matrix $I_{d \times d}$).

Next, we will need to work out the constants $C_0, C_1, D_1, D_2, C_2$ needed in Theorem \ref{thm:global-error}. We have $C_0 = \frac{\sqrt{m}}{2}$, implied from Lemma \ref{lemma:Z}.

The local strong error and local weak error are bounded in Lemma \ref{lemma:local-strong-error} and \ref{lemma:local-weak-error} respectively. Note that the coefficient $\widetilde{C}_1$/$\widetilde{C}_2$ in the bound for local strong/weak error depends on initial value, which changes from iteration to iteration. Combined with Lemma \ref{lemma:bounded-iterates}, we would obtain $C_1$ and $C_2$, namely
\begin{align*}
    \widetilde{C}_1 \le 2(L^2 + G) \left(\frac{d}{4\kappa L} + \mathbb{E}\norm{\bs{x}_0}^2 + \frac{8d}{7m} + 1\right)^\frac{1}{2} 
    \le 2(L^2  + G) \sqrt{\frac{2d}{m} + \mathbb{E}\norm{\bs{x}_0}^2 +  1}  \triangleq C_1
\end{align*}
and 
\begin{equation*}
    \widetilde{C}_2 \leq2L \left( d + \frac{m}{2}\left(\mathbb{E}\norm{\bs{x}_0}^2 + \frac{8d}{7m}\right) \right)^\frac{1}{2} \leq 2L\sqrt{m} \sqrt{\frac{2d}{m} + \mathbb{E}\norm{\bs{x}_0}^2 + 1} \triangleq C_2.
\end{equation*}

We collect all constants here in the proof for easier reference
\begin{align*}
    &A = I_{d\times d}, \kappa_A = 1, \beta = m, \, h_0 = \frac{1}{4\kappa L}, \, C_0 = \frac{\sqrt{m}}{2}, \\
    &C_1 = 2(L^2 + G)  \sqrt{\frac{2d}{m} + \mathbb{E}\norm{\bs{x}_0}^2 +  1} , \, D_1 = 0 \\
    &C_2 = 2L\sqrt{m} \sqrt{\frac{2d}{m} + \mathbb{E}\norm{\bs{x}_0}^2 + 1}, \, D_2 = 0.
\end{align*}
Then the constant in Theorem \ref{thm:global-error} for LMC algorithm simplifies to 
\begin{align*}
    C &= \frac{2}{\sqrt{\beta}}\left( \frac{C_1 + C_0 C_2}{\sqrt{\beta}} +  C_2 \right), \\
    &\leq \frac{10(L^2 + G)}{m^\frac{3}{2}} \sqrt{2d + m\left(\mathbb{E}\norm{\bs{x}_0}^2 + 1\right)} \triangleq \Clmc.
\end{align*}
Assuming $L, m, G$ are all constants and independent of $d$, then clearly $\Clmc = \mathcal{O}(\sqrt{d})$. Then applying Theorem \ref{thm:w2-non-asymptotic} to LMC, we have
\begin{equation}\label{eq:w2_lmc}
    W_2(\text{Law}(\bar{\bs{x}}_k), \mu)
    \leq e^{- mkh}  W_2(\text{Law}(\bs{x}_0), \mu) + \Clmc h
\end{equation}
for $0 < h \leq\frac{1}{4\kappa L}$. 
\end{proof}

\subsection{Proof of Theorem \ref{thm:lower-bound} (Lower Bound of Mixing Time)}
\begin{proof}
If we start from $\bs{x}_0 = \bs{1}_{2d}$ and run LMC for the potential function in Eq. \eqref{eq:quadratic}, we then have
\[
    \left(\bar{\bs{x}}_k\right)_i
    =
    \begin{cases}
    (1 - mh)^k (\bs{x}_0)_i + \sqrt{2h} \sum_{l=1}^k (1-mh)^{k-l} (\bs{\xi}_l)_i  , \, 1 \leq i\le d \\
    (1 - Lh)^k (\bs{x}_0)_i + \sqrt{2h} \sum_{l=1}^k (1-Lh)^{k-l} (\bs{\xi}_l)_i , \, d + 1\le i\le 2d
    \end{cases}
\]
and hence 
\[
    \left(\bar{\bs{x}}_k\right)_i
    \sim 
    \begin{cases}
    \mathcal{N}\left( (1 - mh)^k, \frac{2}{m(2-mh)} \left(1 - (1-mh)^{2k}\right) \right), \, 1\le i\le d \\
    \mathcal{N}\left( (1 - Lh)^k, \frac{2}{L(2-Lh)} \left(1 - (1-Lh)^{2k}\right) \right), \, d + 1\le i\le 2d
    \end{cases}  
\]
Clearly, stability requires $h < \frac{2}{L}$.

The squared 2-Wasserstein distance between the law of the $k$-th iterate of LMC and target distribution is 
\begin{align*}
    W_2^2(\text{Law}(\bar{\bs{x}}_k), \mu)
    =& d(1-mh)^{2k} + \frac{d}{m}\left( \sqrt{\frac{2}{2-mh}}\sqrt{1 - (1-mh)^{2k}} - 1\right)^2 \\
    +& d(1-Lh)^{2k} + \frac{d}{L}\left( \sqrt{\frac{2}{2-Lh}}\sqrt{1 - (1-Lh)^{2k}} - 1\right)^2.
\end{align*}

Suppose $W_2(\text{Law}(\bar{\bs{x}}_k), \mu) \le \epsilon$, we then must have
\begin{align}
    d(1 - mh)^{2k} \le& \epsilon^2 \label{eq:inequality_1}\\
    \frac{d}{m}\left( \sqrt{\frac{2}{2-mh}}\sqrt{1 - (1-mh)^{2k}} - 1\right)^2 \le& \epsilon^2 \label{eq:inequality_2}.
\end{align}

A necessary condition of Eq. \eqref{eq:inequality_2} is that
\begin{equation}\label{eq:inequality_2_necessary}
    1 + \frac{\sqrt{m}}{\sqrt{d}} \epsilon \ge \sqrt{\frac{2}{2-mh}}\sqrt{1 - (1-mh)^{2k}} 
    \stackrel{(i)}{\ge} \sqrt{\frac{2}{2-mh}}\sqrt{1 - \frac{\epsilon^2}{d}}
\end{equation}
where $(i)$ is due to Eq. \eqref{eq:inequality_1}. It follows from Eq. \eqref{eq:inequality_2_necessary} and $m=1$ that
\begin{equation}\label{eq:upper-bound-h}
    h \le \frac{4}{1 + \frac{\epsilon}{\sqrt{d}}} \frac{\epsilon}{\sqrt{d}} \le \frac{4\epsilon}{\sqrt{d}}.
\end{equation}
Revisiting Eq. \eqref{eq:inequality_1} yields
\begin{align}
    &\epsilon^2 \ge d(1-mh)^{2k}
    \stackrel{(i)}{\ge} d \left(1 - 2mh + \frac{(2mh)^2}{2}\right)^{2k}
    \stackrel{(ii)}{\ge} d e^{-4m kh} \nonumber\\
    \iff& k \ge \frac{1}{2hm}\log\frac{\sqrt{d}}{\epsilon}      \label{eq:lower-bound-k}
\end{align}
where $(i)$ is due to $mh < \frac{2}{\kappa} < \frac{1}{2}$ and $(ii)$ is due to $e^{-x} \le 1 - x + \frac{x^2}{2}, 0 < x < 1$.

Combine Eq. \eqref{eq:upper-bound-h} and \eqref{eq:lower-bound-k}, we then obtain a lower bound of the mixing time
\[
    k \ge \frac{\sqrt{d}}{8m\epsilon}\log\frac{\sqrt{d}}{\epsilon} = \frac{\sqrt{d}}{8\epsilon}\log\frac{\sqrt{d}}{\epsilon} = \widetilde{\Omega}\left( \frac{\sqrt{d}}{\epsilon} \right).
\]
\end{proof}

\section{Some Properties of Langevin Dynamics} \label{sec:app-ld}
\subsection{Contraction of Langevin Dynamics}
\begin{lemma}\label{lemma:contraction}
Suppose Assumption \ref{asp:smooth-and-strongly-convex} holds. Then two copies of overdamped Langevin dynamics have the following contraction property
\[
     \left\{\mathbb{E}\norm{\bs{y}_t - \bs{x}_t}^2 \right\}^\frac{1}{2} \le  \left\{ \mathbb{E}\norm{\bs{y} - \bs{x}}^2 \right\}^\frac{1}{2} \exp(-mt) 
\]
where $\bs{x}, \bs{y}$ are the initial values of $\bs{x}_t, \bs{y}_t$.
\end{lemma}
\begin{proof}
First assume $\bs{x}, \bs{y}$ are deterministic. Suppose $\bs{x}_t, \bs{y}_t$ are respectively the solutions to 
\begin{align*}
    d\bs{x}_t =& -\nabla f(\bs{x}_t) dt + \sqrt{2} d\bs{B}_t \\
    d\bs{y}_t =& -\nabla f(\bs{y}_t) dt + \sqrt{2} d\bs{B}_t
\end{align*}
where $\bs{B}_t$ is a standard $d$-dimensional Brownian motion. Denote $L_t = \frac{1}{2} \mathbb{E} \norm{\bs{y}_t - \bs{x}_t}^2$ and take time derivative, we obtain
\begin{align*}
    \frac{d}{dt} L_t =  -\mathbb{E} \innerprod{\bs{y}_t - \bs{x}_t}{\nabla f(\bs{y}_t) - \nabla f(\bs{x}_t)}
    \stackrel{(i)}{\le} - m \mathbb{E} \norm{\bs{y}_t - \bs{x}_t}^2 
    = -2m L_t
\end{align*}
where $(i)$ is due to the strong-convexity assumption made on $f$. We then obtain 
$
    L_t \le L_0 \exp(-2mt)
$
and it follows by Gronwall's inequality that
\[
    \left\{\mathbb{E}\norm{\bs{y}_t - \bs{x}_t}^2 \right\}^\frac{1}{2} \le \norm{\bs{y} - \bs{x}} \exp(-mt).
\]
When $\bs{x}, \bs{y}$ are random, by the conditioning version of the above inequality and Jensen's inequality, we have
\[
    \left\{ \mathbb{E}\left[\mathbb{E}\norm{\bs{y}_t - \bs{x}_t}^2  \bigg\rvert \bs{x}, \bs{y}\right]  \right\}^\frac{1}{2} \le \left\{ \mathbb{E}\norm{\bs{y} - \bs{x}}^2 \exp(-2mt) \right\}^\frac{1}{2} = \left\{ \mathbb{E}\norm{\bs{y} - \bs{x}}^2 \right\}^\frac{1}{2} \exp(-mt)  .
\]
\end{proof}

\subsection{Growth Bound of Langevin Dynamics}
\begin{lemma}\label{lemma:growth}
Suppose Assumption \ref{asp:smooth-and-strongly-convex} holds, then when $0 \le h \le \frac{1}{4\kappa L}$,  the solution of overdamped Langevin dynamics $\bs{x}_t$ satisfies
\[
\mathbb{E} \norm{\bs{x}_h - \bs{x} }^2 \le 6\left( d + \frac{m}{2}\mathbb{E}\norm{\bs{x}}^2 \right) h
\]
where $\bs{x}$ is the initial value at $t=0$.
\end{lemma}
\begin{proof}
We have
\begin{align*}
    \mathbb{E}\norm{\bs{x}_h - \bs{x}}^2 =& \mathbb{E} \norm{ -\int_0^h \nabla f(\bs{x}_t)dt + \sqrt{2}\int_0^h d\bs{B}_t }^2\\
    \le& 2\mathbb{E}\norm{\int_0^h \nabla f(\bs{x}_t)dt}^2 + 4\mathbb{E} \norm{\int_0^h d\bs{B}_t}^2 \\
    \stackrel{(i)}{=}& 2\mathbb{E}\norm{\int_0^h \nabla f(\bs{x}_t)dt}^2 + 4hd \\
    \le& 2\mathbb{E}\left[ \left(\int_0^h \norm{\nabla f(\bs{x}_t) - \nabla f(\bs{x})} dt + \int_0^h \norm{\nabla f(\bs{x})} dt \right)^2 \right] + 4hd \\
    \le& 2\mathbb{E}\left[ \left(L \int_0^h \norm{\bs{x}_t - \bs{x}} dt + h\norm{\nabla f(\bs{x})}\right)^2 \right] + 4hd \\
    \le& 4\mathbb{E}\left[ L^2\left( \int_0^h \norm{\bs{x}_t - \bs{x}} dt \right)^2 + h^2\norm{\nabla f(\bs{x})}^2 \right] + 4hd \\ 
    \stackrel{(ii)}{\le}& 4hd + 4h^2\mathbb{E}\norm{\nabla f(\bs{x})}^2 + 4L^2 h \int_0^h \mathbb{E}\norm{\bs{x}_t - \bs{x} }^2dt
\end{align*}
where $(i)$ is due to Ito's isometry, $(ii)$ is due to Cauchy-Schwarz inequality. By Gronwall's inequality, we obtain
\[
    \mathbb{E} \norm{\bs{x}_h - \bs{x}}^2 \le 4h\left( d + h\mathbb{E}\norm{\nabla f(\bs{x})}^2 \right) \exp\left\{ 4L^2h^2 \right\}.
\]
Since $\norm{\nabla f(\bs{x})} = \norm{\nabla f(\bs{x}) - \nabla f(\bs{0})} \le L \norm{\bs{x}}$, when $0 < h \le \frac{1}{4\kappa L}$, we finally reach at 
\[
    \mathbb{E} \norm{\bs{x}_h - \bs{x}}^2 \le 4 e^\frac{1}{4} \left( d + 2hL^2\mathbb{E}\norm{\bs{x}}^2 \right) h \le 6\left( d + \frac{m}{2}\mathbb{E}\norm{\bs{x}}^2 \right) h.
\]
\end{proof}

\subsection{Bound on Evolved Deviation}
\begin{lemma}\label{lemma:Z}
    Suppose Assumption \ref{asp:smooth-and-strongly-convex} holds. Let $\bs{x}_t, \bs{y}_t$ be two solutions of overdamped Langevin dynamics starting from $\bs{x}, \bs{y}$ respectively, for $0 < h \le \frac{1}{4\kappa L}$, we have the following representation
        \[
            \bs{x}_h - \bs{y}_h = \bs{x} - \bs{y} + \bs{z}
        \]
        with 
        \[
            E\norm{\bs{z}}^2 \le \frac{m}{4} \mathbb{E}\norm{\bs{x} - \bs{y}}^2 h.
        \]
\end{lemma}
\begin{proof}
Let $\bs{z} = (\bs{x}_h - \bs{y}_h) - (\bs{x} - \bs{y}) = -\int_0^h \nabla f(\bs{x}_s) - \nabla f(\bs{y}_s) ds$.
Ito's lemma readily implies that
\begin{align*}
    \mathbb{E}\norm{\bs{x}_h - \bs{y}_h}^2 =& \mathbb{E} \norm{\bs{x} - \bs{y}}^2 -  2 \mathbb{E} \int_0^h \innerprod{\bs{x}_s - \bs{y}_s}{ \nabla f(\bs{x}_s) - \nabla f(\bs{y}_s)} ds \\
    \stackrel{(i)}{\le}& \mathbb{E} \norm{\bs{x} - \bs{y}}^2 - 2m \int_0^h \mathbb{E} \norm{\bs{x}_s - \bs{y}_s}^2 ds \\
    \le& \mathbb{E} \norm{\bs{x} - \bs{y}}^2
\end{align*}
where $(i)$ is due to strong-convexity of $f$. We then have that 
\begin{align*}
    \mathbb{E}\norm{\bs{z}}^2 =&  \norm{ \mathbb{E} \left[\int_0^h \nabla f(\bs{x}_s) - \nabla f(\bs{y}_s) ds \right] }^2 \\
    \le&  \left( \int_0^h \norm{\mathbb{E} \left[\nabla f(\bs{x}_s) - \nabla f(\bs{y}_s) \right]} ds \right) ^2  \\
    \le&  \int_0^h 1^2 ds \int_0^h \norm{\mathbb{E} \left[\nabla f(\bs{x}_s) - \nabla f(\bs{y}_s)\right]}^2 ds \\
    \le& h  \int_0^h \mathbb{E}  \norm{\nabla f(\bs{x}_s) - \nabla f(\bs{y}_s)}^2 ds   \\
    \le& L^2 h \int_0^h \mathbb{E} \norm{\bs{x}_s - \bs{y}_s}^2 ds \\
    \le& L^2 \mathbb{E} \norm{\bs{x} - \bs{y}}^2 h^2 \\
    \stackrel{(i)}{\le}& \frac{m}{4} \mathbb{E} \norm{\bs{x} - \bs{y}}^2 h
\end{align*}
where $(i)$ is due to $h \le \frac{1}{4\kappa L}$.
\end{proof}

\section{Some Properties of LMC Algorithm} \label{sec:app-lmc}
\subsection{Local Strong Error}
\begin{lemma}\label{lemma:local-strong-error}
Suppose Assumption \ref{asp:smooth-and-strongly-convex} holds. Denote the one-step iteration of LMC algorithm with step size $h$ by $\bar{\bs{x}}_1$ and the solution of overdamped Langevin dynamics at time $t=h$ by $\bs{x}_h$. Both the discrete algorithm and the continuous dynamics start from the same initial value $\bs{x}$. If $0\le h \le \frac{1}{4\kappa L}$, then the local strong error of LMC algorithm satisfies
\[
    \left\{\mathbb{E} \norm{\bar{\bs{x}}_1 - \bs{x}_h}^2\right\}^\frac{1}{2} \le \widetilde{C}_2 h^\frac{3}{2}
\]
with $\widetilde{C}_2 = 2L \left( d + \frac{m}{2}\mathbb{E}\norm{\bs{x}}^2 \right)^\frac{1}{2}$.
\end{lemma}
\begin{proof}
We have for $0 \le h \le \frac{1}{4\kappa L}$,
\begin{align*}
    \mathbb{E}\norm{\bar{\bs{x}}_1 - \bs{x}_h}^2 =& \mathbb{E} \norm{\int_0^h \nabla f(\bs{x}_s) - \nabla f(\bs{x}) ds}^2 \\
    \le& \mathbb{E} \left( \int_0^h \norm{\nabla f(\bs{x}_s) - \nabla f(\bs{x}) }ds \right)^2 \\
    \le& L^2 \mathbb{E} \left( \int_0^h \norm{\bs{x}_s - \bs{x}} ds \right)^2 \\
    \stackrel{(i)}{\le}& L^2h \int_0^h  \mathbb{E}  \norm{\bs{x}_s - \bs{x}}^2 ds \\
    \stackrel{(ii)}{\le}& 3L^2 \left( d + \frac{m}{2}\mathbb{E}\norm{\bs{x}}^2 \right) h^3
\end{align*}
where $(i)$ is due to Cauchy-Schwartz inequality and $(ii)$ is due to Lemma \ref{lemma:growth}. Taking square roots on both side completes the proof.
\end{proof}

\subsection{Local Weak Error}
\begin{lemma}\label{lemma:local-weak-error}
Suppose Assumption \ref{asp:smooth-and-strongly-convex} and \ref{asp:linear_3rd_derivative} hold. Denote the one-step iteration of LMC algorithm with step size $h$ by $\bar{\bs{x}}_1$ and the solution of overdamped Langevin dynamics at time $t=h$ by $\bs{x}_h$.  Both the discrete algorithm and the continuous dynamics start from the same initial value $\bs{x}$. If $0\le h \le \frac{1}{4\kappa L}$, then the local weak error of LMC algorithm satisfies
\[
     \norm{ \mathbb{E}\bar{\bs{x}}_1 - \mathbb{E}\bs{x}_h} \le \widetilde{C}_1 h^2
\]
with $\widetilde{C}_1 = 2(L^2 + G) \left(\frac{d}{4\kappa L} + \mathbb{E}\norm{\bs{x}}^2 + 1\right)^\frac{1}{2}$.
\end{lemma}
\begin{proof}
By Ito's lemma, we have
\[
    d\nabla f(\bs{x}_t) = -\nabla^2 f(\bs{x}_t)\nabla f(\bs{x}_t) dt + \nabla (\Delta f(\bs{x}_t)) dt + \sqrt{2} \int_0^t \nabla^2 f(\bs{x}_t) d\bs{B}_t.
\]
It follows that 
\begin{align*}
    \norm{\mathbb{E}\bar{\bs{x}}_1 - \mathbb{E}\bs{x}_h} =& \norm{\mathbb{E} \int_0^h \nabla f(\bs{x}_s) - \nabla f(\bs{x}) ds} \\
    =& \norm{\mathbb{E} \left\{\int_0^h \int_0^s  -\nabla^2 f(\bs{x}_r) \nabla f(\bs{x}_r) + \nabla (\Delta f(\bs{x}_r)) dr   ds + \sqrt{2}\int_0^h \int_0^s \nabla^2 f(\bs{x}_r) d\bs{B}_r  ds\right\} } \\
    =& \norm{\mathbb{E} \left\{\int_0^h \int_0^s   -\nabla^2 f(\bs{x}_r) \nabla f(\bs{x}_r) + \nabla (\Delta f(\bs{x}_r)) dr   ds \right\} }\\
    \le&  \int_0^h \int_0^s \mathbb{E}\norm{\nabla^2 f(\bs{x}_r) \nabla f(\bs{x}_r)} dr ds  + \int_0^h \int_0^s \mathbb{E}\norm{ \nabla (\Delta f(\bs{x}_r))} dr ds \\
    \le& L \int_0^h \int_0^s \mathbb{E} \norm{\nabla f(\bs{x}_r)} dr ds  + \int_0^h \int_0^s \mathbb{E} \norm{ \nabla (\Delta f(\bs{x}_r))} dr ds \\
    \stackrel{(i)}{\le}& (L^2 + G)  \int_0^h \int_0^s \mathbb{E} \norm{\bs{x}_r} dr ds  + \frac{G}{2}h^2 \\
    \le& (L^2 + G)  \left(\int_0^h \int_0^s \mathbb{E} \norm{\bs{x}_r - \bs{x}} dr ds + \frac{h^2}{2}\mathbb{E}\norm{\bs{x}}\right) + \frac{G}{2}h^2\\
    \stackrel{(ii)}{\le}& (L^2 + G)  \left( \int_0^h \int_0^s \sqrt{\mathbb{E} \norm{\bs{x}_r - \bs{x}}^2} dr ds + \frac{h^2}{2}\mathbb{E}\norm{\bs{x}} \right)  + \frac{G}{2}h^2\\
    \stackrel{(iii)}{\le}& (L^2 + G)  \left( \int_0^h \int_0^s \sqrt{6\left( d + \frac{m}{2}\mathbb{E}\norm{\bs{x}}^2 \right) r}  dr ds + \frac{h^2}{2}\mathbb{E}\norm{\bs{x}}  \right) + \frac{G}{2}h^2\\
    =& (L^2 + G)  \left( \frac{4\sqrt{6}}{15} \sqrt{\left( d + \frac{m}{2}\mathbb{E}\norm{\bs{x}}^2 \right)h} + \frac{1}{2}\mathbb{E}\norm{\bs{x}}  \right) h^2 + \frac{G}{2}h^2\\
    \stackrel{(iv)}{\leq}& (L^2 + G) h^2 \sqrt{\left( d + \frac{m}{2}\mathbb{E}\norm{\bs{x}}^2 \right)h + \frac{1}{2}\mathbb{E}\norm{\bs{x}}^2 } + \frac{G}{2}h^2\\
    \stackrel{(v)}{\leq}& (L^2 + G) h^2 \sqrt{\frac{d}{4\kappa L} + \mathbb{E}\norm{\bs{x}}^2 } + \frac{G}{2}h^2\\
    \leq& (L^2 + G) \left( \sqrt{\frac{d}{4\kappa L} + \mathbb{E}\norm{\bs{x}}^2 } + 1 \right) h^2 \\
    \leq& 2(L^2 + G) \left(\frac{d}{4\kappa L} + \mathbb{E}\norm{\bs{x}}^2 + 1\right)^\frac{1}{2} h^2
\end{align*}
where $(i)$ is due to Assumption \ref{asp:linear_3rd_derivative}, $(ii)$ is due to Jensen's inequality, $(iii)$ is due to Lemma \ref{lemma:growth}, $(iv)$ is due to $\sqrt{a} + \sqrt{b} \leq \sqrt{2}\sqrt{a^2 + b^2}$ and $(v)$ is due to $h \leq \frac{1}{4\kappa L}$. It is worth noting in the third equation that the Ito's correction term $\nabla \Delta f$ can also be written as $\Delta \nabla f$ as the two operators commute for $\mathcal{C}^3$ functions.
\end{proof}

\subsection{Boundedness of LMC Algorithm}
\begin{lemma}\label{lemma:bounded-iterates}
Suppose Assumption \ref{asp:smooth-and-strongly-convex}  holds. Denote the iterates of LMC by $\bar{\bs{x}}_k$. If $0\le h \le \frac{1}{4\kappa L}$ we then have the iterates of LMC algorithm are uniformly upper bounded by
\[
    \mathbb{E}\norm{\bar{\bs{x}}_k}^2 \le \mathbb{E}\norm{\bs{x}_0}^2 + \frac{8d}{7m}, \quad \forall k \ge 0
\]
\end{lemma}
%\tao{how general can we make this lemma be (beyond LMC)?} \li{I am afraid this lemma is LMC-specific, }
\begin{proof}
We have
\begin{align*}
    \mathbb{E} \norm{\bar{\bs{x}}_{k+1}}^2 =& \mathbb{E} \norm{\bar{\bs{x}}_k - h \nabla f(\bar{\bs{x}}_k)  +  \sqrt{2h} \bs{\xi}_{k+1}}^2 \\
    \stackrel{(i)}{=}& \mathbb{E}\norm{\bar{\bs{x}}_k}^2 + h^2\mathbb{E}\norm{\nabla f(\bar{\bs{x}}_k)}^2 + 2hd - 2h\mathbb{E}\innerprod{\bar{\bs{x}}_k}{\nabla f(\bar{\bs{x}}_k)} \\
    =& \mathbb{E}\norm{\bar{\bs{x}}_k}^2 + h^2\mathbb{E}\norm{\nabla f(\bar{\bs{x}}_k) - \nabla f(0)}^2 + 2hd - 2h\mathbb{E}\innerprod{\bar{\bs{x}}_k}{\nabla f(\bar{\bs{x}}_k)} \\
    \stackrel{(ii)}{\le}& \mathbb{E}\norm{\bar{\bs{x}}_k}^2 + h^2L^2\mathbb{E}\norm{\bar{\bs{x}}_k}^2 + 2hd - 2h\mathbb{E}\innerprod{\bar{\bs{x}}_k}{\nabla f(\bar{\bs{x}}_k)} \\
    \stackrel{(iii)}{\le}& \mathbb{E}\norm{\bar{\bs{x}}_k}^2 + h^2L^2\mathbb{E}\norm{\bar{\bs{x}}_k}^2 + 2hd - 2mh\mathbb{E}\norm{\bar{\bs{x}}_k}^2 \\
    \stackrel{(iv)}{\le}& \left(1 - \frac{7}{4}mh\right)\mathbb{E}\norm{\bar{\bs{x}}_k}^2 +2hd \\
\end{align*}
where $(i)$ is due to the independence between $\bs{\xi}_{k+1}$ and $\bar{\bs{x}}_k$, $(ii)$ is due to Assumption \ref{asp:smooth-and-strongly-convex}, $(iii)$ is due to the property of $m$-strongly-convex functions, $\innerprod{\nabla f(\bs{y}) - \nabla f(\bs{x})}{\bs{y} - \bs{x}} \ge m \norm{\bs{y} - \bs{x}}^2 \, \forall \bs{x},\bs{y} \in \mathbb{R}^d$, and $(iv)$ uses the assumption $h \le \frac{1}{4\kappa L}$.

Unfolding the inequality, we obtain
\[
    \mathbb{E}\norm{\bar{\bs{x}}_k}^2 \le (1-\frac{7}{4}mh)^k \mathbb{E}\norm{\bar{\bs{x}}_0}^2 + 2hd\left(1 + \frac{7}{4}mh + \cdots + (\frac{7}{4}mh)^{k-1}\right) \le \mathbb{E} \norm{\bs{x}_0}^2 + \frac{8d}{7m}
\]
\end{proof}

\end{document}